\documentclass{article} 

\PassOptionsToPackage{table}{xcolor}
\usepackage{iclr2027_conference,times}

\usepackage{amsmath,amsfonts,bm}

\def\eqref#1{equation~\ref{#1}}

\def\1{\bm{1}}

\DeclareMathAlphabet{\mathsfit}{\encodingdefault}{\sfdefault}{m}{sl}
\SetMathAlphabet{\mathsfit}{bold}{\encodingdefault}{\sfdefault}{bx}{n}

\usepackage{hyperref}
\usepackage{url}
\usepackage{graphicx}
\usepackage{amsthm,needspace}
\usepackage{booktabs,tabularx,multirow,placeins}
\usepackage[table]{xcolor}
\usepackage[most]{tcolorbox}
\newtheorem{lemma}{Lemma}[section]
\newtheorem{theorem}{theorem}[section]

\definecolor{frameworkink}{RGB}{38,57,76}
\definecolor{frameworktint}{RGB}{241,245,249}
\definecolor{frameworkforward}{HTML}{0072B2}
\definecolor{frameworkreverse}{HTML}{B45309}
\definecolor{frameworkpolicy}{HTML}{8E4585}
\newcommand{\forwardKL}{\textcolor{frameworkforward}{Forward KL}}
\newcommand{\reverseKL}{\textcolor{frameworkreverse}{Reverse KL}}
\newcommand{\targetpolicy}{\textcolor{frameworkpolicy}{\pi_\theta}}
\newcommand{\algorithmgroup}[1]{%
  \specialrule{\lightrulewidth}{2pt}{0pt}%
  \multicolumn{4}{l}{\textcolor{frameworkink}{\textbf{#1}}}\\%
  \specialrule{\lightrulewidth}{0pt}{2pt}}

\definecolor{takeawayaccent}{HTML}{3D647A}
\definecolor{takeawaybackground}{HTML}{F1F5F8}
\newtcolorbox{insight}{
  enhanced,
  colback=takeawaybackground,
  coltext=black,
  frame hidden,
  borderline west={1.5pt}{0pt}{takeawayaccent},
  sharp corners,
  boxrule=0pt,
  boxsep=0pt,
  left=10pt, right=10pt, top=8pt, bottom=8pt,
  before skip=10pt, after skip=10pt,
  fontupper=\normalfont\small,
  halign=left,
  before upper={\textcolor{takeawayaccent}{\bfseries Insight.}\enspace},
}

\title{A Distribution-Matching Perspective on LLM Training Algorithms}

\author{%
Yanxing Huang$^{1,3}$,\,Bingzhang Yan$^{2,*}$,\,Zihan Tang$^{1,*}$,\,Wenzhe Li$^{5}$,\,Peng Li$^{3,\dagger}$,\,Yang Liu$^{3,4,\dagger}$\\[4pt]
{\normalfont\small $^{1}$Qiuzhen College, Tsinghua University, Beijing, China}\\
{\normalfont\small $^{2}$Weiyang College, Tsinghua University, Beijing, China}\\
{\normalfont\small $^{3}$Institute for AI Industry Research (AIR), Tsinghua University, Beijing, China}\\
{\normalfont\small $^{4}$Dept.\ of Comp.\ Sci\ \& Tech., Institute for AI, Tsinghua University, Beijing, China}\\
{\normalfont\small $^{5}$Department of Mathematical Science, Tsinghua University, Beijing, China}\\[3pt]
{\normalfont\small $^{*}$Equal contribution.\quad $^{\dagger}$Co-corresponding authors.}\\
{\normalfont\small\hypersetup{hidelinks}%
\href{mailto:hyx25@mails.tsinghua.edu.cn}{\texttt{hyx25@mails.tsinghua.edu.cn}}\quad
\href{mailto:lipeng@air.tsinghua.edu.cn}{\texttt{lipeng@air.tsinghua.edu.cn}}\quad
\href{mailto:liuyang2011@tsinghua.edu.cn}{\texttt{liuyang2011@tsinghua.edu.cn}}}%
}

\iclrfinalcopy 
\begin{document}

\maketitle
\fancyhead[L]{Preprint}

\begin{abstract}
In this work, we propose a unified framework for training algorithms across different stages of modern large language model development, based on the perspective of distribution matching. We show that a broad class of language model training algorithms can be interpreted as minimizing a certain divergence between the current model and a target distribution, and that characterizing the properties of this target distribution provides substantial insight into the behavior of the corresponding training algorithms.
By analyzing specific forms of the target distribution, we establish that On-Policy Distillation (OPD) enjoys an advantage over classical Knowledge Distillation (KD) when the current model distribution is close to the target distribution, while exhibiting a disadvantage when the two distributions are far apart. Furthermore, when the target distribution itself depends on the model being optimized, we show that the accumulation of estimation variance can induce a deterministic degradation of the target distribution.
These results provide a theoretical explanation, at least in part, for the long-term behavior of existing Reinforcement Learning (RL) algorithms. With this intuition we further designed Conditional KL (CKL) regularization for mitigating the degradation of the target distribution. CKL can effectively preserve mode diversity while relaxing the constraint imposed by classical KL regularization on reward improvements. We believe this theoretical perspective can bring some new understanding of language model training algorithms and provide useful guidance for future research.
\end{abstract}

\section{Introduction}

Modern Large Language Models (LLMs) require a complex training pipeline with different algorithms from pre-training to post-training~\citep{team_kimi_2026, deepseek-ai_deepseek-v4_2026, glm-5-team_glm-5_2026}. These algorithms are motivated by different intuitions. For example, the objectives of pre-training and Supervised Fine-Tuning (SFT) arise directly from distribution matching under the cross-entropy loss; On Policy Distillation (OPD) is constructed around on-policy sampling~\citep{lochab_uniform-correct_2026}; Reinforcement Learning (RL) aims to maximize the expected reward~\citep{williams_simple_1992}; and MaxRL can be viewed as arising from maximum-likelihood estimation of correct solutions~\citep{tajwar_maximum_2026}. We hope to identify the commonalities and essential differences among these algorithms, which will provide a more unified understanding and insight into the study of their properties.

Many of the training algorithms of LLMs already admit a distribution matching interpretation~\citep{korbak_reinforcement_2022, khalifa_distributional_2021}, whereas RL is often viewed fundamentally distinct. In this work we first studied the property of an ideal update step in RL and show that, under binary rewards, it preserves the conditional distribution associated with a given reward value. Moreover, if we consider the forward KL divergence between the current model and the true conditional distribution, we will obtain the training objective of MaxRL~\citep{tajwar_maximum_2026}, which is closely related to RL. Similar results can also be extended to RL with nonbinary rewards.

This provides a unified perspective that places the mainstream training algorithms for LLMs within a common distribution-matching framework. Analyzing how the target distribution changes and affects the learning dynamics under different conditions can yield a variety of interesting theoretical results that characterize the properties of these algorithms.

Depending on whether the target distribution being optimized depends on the current model distribution, these algorithms can be further categorized into iterative and non-iterative methods. In what follows, we progressively examine the properties of three classes of algorithms, from simpler to more complex cases: non-iterative methods, iterative methods with zero-mean gradients, and iterative methods with nonzero-mean gradients.

For the first class, we mainly study OPD and SFT-based knowledge distillation~\citep{hinton_distilling_2015} algorithms. By analyzing the variance induced by different target distributions under different assumptions, we roughly classified the boundary of effectiveness of these algorithms, providing a theoretical explanation for previous empirical observations regarding the stability of distillation algorithms~\citep{li_rethinking_2026, oh_kl_2026}.

The second class primarily consists of the iterative self-distillation algorithms. These methods lie between distillation and RL, while retaining theoretical properties that are more amenable to analysis. Within this class, we identify an intriguing \emph{variance accumulation} phenomenon in iterative training, which leads to a monotonic decrease in the expected entropy of the target distribution. In contrast to previous local results derived from training dynamics~\citep{cui_entropy_2025, ren_learning_2025, jin_revisiting_2026}, our analysis further shows that, although the model entropy may either increase or decrease after an individual update, its long-term behavior exhibits a systematic downward trend under the variance accumulation.

The third class mainly concerns of RL algorithms. In realistic training settings, however, the target distribution invariance of RL and MaxRL under idealized conditions is no longer preserved. Their target distributions instead exhibit collapse, driven in part by effects such as variance accumulation. Such target distribution drift serves as an important manifestation of the discrepancy between practical training dynamics and their idealized counterparts. inspired by the analysis on the target distribution, we explored Conditional KL (CKL) regularization in RL. CKL regularization directly controls the drift of the target distribution in RL, preserving the diversity of correct samples with relaxed restriction on reward gains.

\section{A Unified Framework of LLM Training Algorithms}

\subsection{Pre-Training And Distillation Algorithms}

In pre-training and Supervised Fine-Tuning (SFT), minimizing the standard cross-entropy loss is equivalent to minimizing the forward KL divergence from the data distribution to the model distribution. To make this precise, let $\mathcal{D}=\{x^{(i)}\}_{i=1}^N$ be a fixed, nonempty dataset of token sequences, let $p_{\mathcal{D}}(x)=N^{-1}\sum_{i=1}^N\mathbf{1}\{x^{(i)}=x\}$ be its empirical distribution, and let $\pi_\theta(x)$ be the model probability of sequence $x$ with parameter $\theta$. Then

\begin{align}
 D_{\mathrm{KL}}(p_{\mathcal{D}}\Vert\pi_\theta)
 &= \mathbb{E}_{x\sim p_{\mathcal{D}}}
    \left[\log\frac{p_{\mathcal{D}}(x)}{\pi_\theta(x)}\right] \notag = -\mathbb{E}_{x\sim p_{\mathcal{D}}}[\log\pi_\theta(x)]
    -H(p_{\mathcal{D}}) \notag\\
 &= \mathcal{L}_{\mathrm{CE}}(\theta)-H(p_{\mathcal{D}}),
 \label{eq:ce-forward-kl}
\end{align}

Here the data entropy $H(p_{\mathcal{D}})=-\mathbb{E}_{x\sim p_{\mathcal{D}}} [\log p_{\mathcal{D}}(x)]$ is finite and independent of $\theta$ which has no effect on the gradient. This proves the equivalence.

For SFT, let $x$ denote a prompt and $y$ its target response, with fixed empirical joint distribution $p_{\mathcal{D}}(x,y)$. The same decomposition holds conditionally:

\begin{align}
  \mathcal{L}_{\mathrm{SFT}}(\theta) &= 
  -\mathbb{E}_{(x,y)\sim p_{\mathcal{D}}}
  [\log\pi_\theta(y\mid x)]\\
  &=\mathbb{E}_{x\sim p_{\mathcal{D}}}
 \left[D_{\mathrm{KL}}\bigl(p_{\mathcal{D}}(\cdot\mid x)
 \Vert\pi_\theta(\cdot\mid x)\bigr)\right]
 +H_{p_{\mathcal{D}}}(Y\mid X),
 \label{eq:sft-conditional-kl}
\end{align}

Under the autoregressive factorization,
$-\log\pi_\theta(y\mid c)=-\sum_{t=1}^{|y|}
\log\pi_\theta(y_t\mid c,y_{<t})$, this sequence-level objective is exactly the sum of token-level
cross-entropies with ground-truth prefixes.

In many applications the response could come from a teacher model $\pi_t$, instead of the dataset. Then the training objective of SFT would become minimizing the forward KL divergence between the teacher model and the student model. This yields the classical SFT-based Knowledge-Distillation (KD) algorithm~\citep{hinton_distilling_2015}, which is often simply called SFT as well:

\begin{align}
  \mathcal{J}_{\mathrm{KD}} &= -\mathbb{E}_{x\sim\mathcal{D}} [D_{\mathrm{KL}}\bigl(\pi_t (\cdot\mid x)\Vert \pi_\theta (\cdot\mid x)\bigr)] = \mathbb{E}_{x\sim\mathcal{D}, y\sim\pi_t (\cdot\mid x)} [\log \pi_\theta (y\mid x)] - H(\pi_t).
  \label{eq:kl-loss}
\end{align}

OPD recently emerged as a new algorithm for distillation tasks~\citep{thinking_machines_lab_-policy_2025}. It can exhibit superior performance than SFT in some cases, and received great attention from the industry~\citep{team_kimi_2026, deepseek-ai_deepseek-v4_2026, glm-5-team_glm-5_2026}. While in some cases this algorithm may also be unstable and failed to complete the distillation task~\citep{li_rethinking_2026}. The key difference between OPD and KD is that, OPD uses sequence-level reverse KL divergence as the training objective:

\begin{align}
  \mathcal{J}_{\mathrm{OPD}} &= -\mathbb{E}_{x\sim\mathcal{D}} [D_{\mathrm{KL}}\bigl(\pi_\theta (\cdot\mid x)\Vert \pi_t (\cdot\mid x)\bigr)]
  = \mathbb{E}_{x\sim\mathcal{D}, y\sim\pi_\theta (\cdot\mid x)} [\log \frac{\pi_t (y\mid x)}{\pi_\theta (y\mid x)}].
  \label{eq:opd-loss}
\end{align}


We should note that the sequence-level KL objective directly yields sampled-token OPD. While there exist several variants of this OPD~\citep{li_rethinking_2026}, this sampled-token OPD is indeed adopted by the industry~\citep{team_kimi_2026} and share some important properties with other variants. the following analysis will primarily be built upon this algorithm.

\subsection{Reinforce Learning Algorithms}


The objective of RL is to maximize the expected reward $r$ given to each possible action of the policy model, while for LLMs an action is typically a sequence generation. Its standard objective and gradient are~\citep{williams_simple_1992}:

\begin{equation}
  \mathcal{J}_{\mathrm{RL}} = \mathbb{E}_{x\sim\mathcal{D}, y\sim\pi_\theta (\cdot\mid x)} [r(x, y)],\qquad \nabla_\theta \mathcal{J}_{\mathrm{RL}} = \mathbb{E}_{x\sim\mathcal{D}, y\sim\pi_\theta (\cdot\mid x)} [r(x, y) \nabla_\theta \log \pi_\theta].
  \label{eq:std-rfalgo}
\end{equation}

Modern reinforcement learning algorithms such as PPO~\citep{schulman_proximal_2017} and GRPO~\citep{shao_deepseekmath_2024} additionally introduce various methods, including advantage estimation and importance sampling techniques, to better achieve stable training, but they are essentially optimizing the same training objective.


In fact, this objective also implicitly introduces a target distribution. We first characterize an ideal policy update that improves expected reward
while remaining close to the current policy under a forward-KL restriction. (Similar result also holds for reverse-KL restriction)

\begin{theorem}[Reward reweighting under a forward-KL restriction]
\label{lem:reward-reweighting}
For a prompt $x$, the ideal update maximizing
$\mathbb E_{y\sim q}[r(x,y)]-\beta D_{\mathrm{KL}}(\pi_\theta\Vert q)$
over response distributions $q$, with KL penalty $\beta>0$, is
\begin{equation}
 q^*(y\mid x)
 =\frac{\beta\pi_\theta(y\mid x)}{\lambda_x-r(x,y)},
 \label{eq:reward-reciprocal-tilt}
\end{equation}
where $\lambda_x$ normalizes the distribution. Responses with the same
reward receive the same probability multiplier, preserving their
conditional distribution.
\end{theorem}
\noindent\hyperref[app:reward-reweighting]{Proof: Appendix~\ref*{app:reward-reweighting}.}

Under binary rewards, we can verify that this update preserves the true conditional distribution
$\pi_\theta^+(\cdot\mid x)
=\pi_\theta(\cdot\mid x,r(x,y)=1)$. If we use this distribution as the training target, matching
this conditional distribution with the forward KL gives the following identity.

\Needspace{9\baselineskip}
\begin{theorem}[True conditional KL and the MaxRL]
\label{lem:conditional-kl-maxrl}
Let $\rho_\theta(x)=\mathbb E_{y\sim\pi_\theta(\cdot\mid x)}[r(x,y)]$
denote the success probability. Matching the true conditional target
$\pi_\theta^+$ by minimizing its forward KL to the current policy yields
\begin{equation}
 -\mathbb E_{x\sim\mathcal D}
 D_{\mathrm{KL}}\bigl(\pi_\theta^+(\cdot\mid x)
                  \Vert\pi_\theta(\cdot\mid x)\bigr)
 =\mathbb E_{x\sim\mathcal D}[\log\rho_\theta(x)]
 =\mathcal J_{\mathrm{MaxRL}}^{(\infty)}(\theta).
 \label{eq:conditional-kl-maxrl}
\end{equation}
\end{theorem}
\noindent\hyperref[app:conditional-kl-maxrl]{Proof: Appendix~\ref*{app:conditional-kl-maxrl}.}

This actually gives the limit objective function of MaxRL algorithm, which is previously derived from the maximum likelihood estimator of the predictive accuracy~\citep{tajwar_maximum_2026}. It is believed to have some superior performance on maintaining the diversity of the distribution. The standard RL objective is the first-order approximation of MaxRL. They induce the same gradient direction for an individual sample, differing only in the average weighting assigned across samples. Since our subsequent analysis focuses on individual samples, the two methods are effectively equivalent in the properties relevant to our analysis.

An auxiliary binary acceptance variable extends the conditional-KL identity to bounded nonbinary rewards on an augmented response space; Appendix~\ref{app:randomized-reward-extension} outlines this construction, while our main RL analysis remains focused on binary rewards.

\begin{table}[!htbp]
\centering
\caption[Training algorithms viewed as distribution matching]{Training algorithms viewed as distribution matching. Parameter-independent constants
are omitted. Conditional objectives are averaged over prompts
$x\sim\mathcal D$. \forwardKL{} places the target in the first argument;
\reverseKL{} places it in the second. RL and MaxRL use binary rewards and
$\rho_\theta(x)>0$, with the policy-dependent target $\pi_\theta^+$.
Algorithms are iterative if their target distribution depends on the current
model, and noniterative otherwise; \textcolor{frameworkpolicy}{$\pi_\theta$}
is highlighted in iterative targets.
$^{\dagger}$\,RL is its first-order approximation.}
\label{tab:distribution-matching}
\vspace{5pt}
\begingroup
\small
\setlength{\tabcolsep}{5pt}
\renewcommand{\arraystretch}{1.45}
\begin{tabularx}{\linewidth}{
 >{\raggedright\arraybackslash}p{0.15\linewidth}
 >{\centering\arraybackslash}X
 >{\centering\arraybackslash}p{0.235\linewidth}
 >{\raggedright\arraybackslash}p{0.16\linewidth}}
\specialrule{\heavyrulewidth}{0pt}{0pt}
\rowcolor{frameworktint}
\textcolor{frameworkink}{\textbf{Algorithm}} &
\textcolor{frameworkink}{\textbf{Objective function}} &
\textcolor{frameworkink}{\textbf{Target distribution}} &
\textcolor{frameworkink}{\textbf{Divergence}} \\
\algorithmgroup{Noniterative algorithms}
Pre-training &
$\mathbb E_{x\sim p_{\mathcal D}}[\log\pi_\theta(x)]$ &
$p_{\mathcal D}(x)$ & \forwardKL \\
SFT &
$\mathbb E_{y\sim p_{\mathcal D}(\cdot\mid x)}[\log\pi_\theta(y\mid x)]$ &
$p_{\mathcal D}(y\mid x)$ & \forwardKL \\
\addlinespace[3pt]
SFT-based KD &
$\mathbb E_{y\sim\pi_t(\cdot\mid x)}[\log\pi_\theta(y\mid x)]$ &
$\pi_t(y\mid x)$ & \forwardKL \\
OPD &
$\displaystyle\mathbb E_{y\sim\pi_\theta(\cdot\mid x)}
 \!\left[\log\frac{\pi_t(y\mid x)}{\pi_\theta(y\mid x)}\right]$ &
$\pi_t(y\mid x)$ & \reverseKL \\
\algorithmgroup{Iterative algorithms}
RL &
$\mathbb E_{y\sim\pi_\theta(\cdot\mid x)}[r(x,y)]$ &
$\displaystyle\targetpolicy^+=\frac{r(x,y)\targetpolicy(y\mid x)}{\rho_\theta(x)}$ &
\forwardKL$^{\dagger}$ \\
MaxRL &
$\log\mathbb E_{y\sim\pi_\theta(\cdot\mid x)}[r(x,y)]$ &
$\displaystyle\targetpolicy^+=\frac{r(x,y)\targetpolicy(y\mid x)}{\rho_\theta(x)}$ &
\forwardKL \\
\bottomrule
\end{tabularx}
\endgroup
\end{table}
\FloatBarrier
%


\section{The Effectiveness Border of Distillation Algorithms}

When OPD algorithm was initially introduced, its superior performance was believed to come from its on policy nature~\citep{thinking_machines_lab_-policy_2025}. SFT trains the student on the trajectory generated by the teacher. Even if the student learns well on prediction over teacher's prefix, if it can not generate a prefix with the same quality, the performance will still be limited. OPD method however, directly uses teacher's predictive distribution to correct the student on its own trajectory, and avoids this drawback~\citep{thinking_machines_lab_-policy_2025}. However, this discussion ignored another important impact in OPD: The teacher's probability prediction on the student's prefix may not be sufficiently reliable as well. In some previous works we did observe some cases where OPD may face stability issue and failed to improve the student model~\citep{li_rethinking_2026}. Clearly this issue does not exist in SFT.

It is difficult to determine precisely which of the two issues discussed above has a dominant impact on performance. Beyond such relatively subjective analysis, however, we aim to establish rigorous theoretical results and characterize the regimes in which the two algorithms are effective. Since all distillation methods have the same target distribution, one primary theoretical difference may lie in the estimation of the gradient. This leads to some interesting results.

\subsection{Proximal Behavior}

We study the proximal regime, where the student distribution is close to the teacher. At the \emph{proximal limit}, they coincide, while the teacher is held fixed during differentiation. Both negative-KL population objectives and their gradients then vanish, but their sampled gradient estimators behave differently.

\Needspace{11\baselineskip}
\begin{lemma}[Gradient noise at the proximal limit]
\label{lem:proximal-gradient-noise}
Using $m$ independent responses to a prompt $x$. At
$\pi_{\theta_0}(\cdot\mid x)=\pi_t(\cdot\mid x)$, both gradient estimators are unbiased with mean zero, but
\begin{equation}
 \widehat g_{\mathrm{OPD}}(\theta_0)=0\quad\text{almost surely},
 \qquad
  \operatorname{Cov}\bigl(\widehat g_{\mathrm{SFT}}(\theta_0)\bigr)
 =\frac1m F_x(\theta_0),
 \label{eq:proximal-gradient-noise}
\end{equation}
where $F_x(\theta_0)$ is the student Fisher information at prompt $x$.
Thus SFT typically has nonzero gradient noise at, whereas OPD vanishes for every sampled batch.
\end{lemma}
\noindent\hyperref[app:proximal-gradient-noise]{Proof: Appendix~\ref*{app:proximal-gradient-noise}.}

Under the smoothness assumptions, this strict gap still persists in a neighborhood of $\theta_0$: OPD has lower mean-squared error in estimating its population gradient than SFT has. This establishes a local advantage in gradient estimation in a range of proximal case, and could provide a theoretical explanation on the superior performance of OPD.

\subsection{Distal Behavior}

The distal regime can exhibit different behavior. When the student assigns appreciable probability to responses that the teacher considers very unlikely, the log-ratio weight $w_\theta(y\mid x)=\log\pi_t(y\mid x)-\log\pi_\theta(y\mid x)$ can amplify OPD's gradient noise. Under specific conditions, OPD's gradient estimation variance can be proved to exceed that of sampled-response SFT by an arbitrarily large factor, see Appendix~\ref{app:distal-gradient-variance}.




These results are also supported by our experiments, see Appendix~\ref{app:distillation-experiments} for more details.

\section{Variance Accumulation in Iterative Algorithms}\label{sec:variance-accumulation}

Distillation algorithms have a fixed target distribution for the student to fit. Although we have showed in Theorem~\ref{lem:reward-reweighting}, the ideal single-step update of RL or MaxRL with binary reward will not affect the target distribution, this property may not hold in real applications. One important reason is that, there are non-negligible random noise in the one optimization step. This could be caused by the empirical estimator of the objective, the approximation of the optimizer and the numerical precision of the model parameter.

In this section we study \emph{variance accumulation} in iterative training and its relation to the observed entropy dynamics, starting from the gradient noise in self-distillation.

\subsection{Iterative Self-Distillation}

We start from \emph{iterative self-distillation}: each step samples responses from the current model and fits this empirical distribution. This is equivalent to an uncentered RL update with reward $1$ for every response, and coincides with sampled-response SFT at the proximal limit each step. The distinction is that, in SFT at the proximal limit, the target distribution is fixed as $\pi_t = \pi_{\theta_0}$, whereas in iterative self-distillation the target distribution is $\displaystyle\targetpolicy$ which updates with the current policy.


By Lemma~\ref{lem:proximal-gradient-noise}, the sampled gradient has zero mean but generally nonzero variance. Fix a prompt $x$, write $\pi_k=\pi_{\theta_k}(\cdot\mid x)$, and let $\mathcal F_k$ denote the training history before sampling at step $k$. Then we have the following result:

\begin{lemma}[Martingale structure of self-distillation]
\label{lem:self-distillation-martingale}
For iterative self-distillation with step size $\eta_k$, the parameters form a martingale. The induced probabilities satisfy the same property to first order:
\begin{equation}
 \mathbb E[\theta_{k+1}\mid\mathcal F_k]=\theta_k,
 \qquad
 \mathbb E[\pi_{k+1}(y)\mid\mathcal F_k]
 =\pi_k(y)+O(\eta_k^2).
 \label{eq:self-distillation-martingale}
\end{equation}
\end{lemma}
\noindent\hyperref[app:variance-accumulation]{Proof: Appendix~\ref*{app:variance-accumulation}.}
%

\begin{theorem}[Conditional decrease of expected entropy]
\label{lem:self-distillation-entropy}
Let $h_x(\theta)=H(\pi_\theta(\cdot\mid x))$. If $h_x$ is concave along every possible update segment from $\theta_k$ to $\theta_{k+1}$, then
\begin{equation}
 \mathbb E[H(\pi_{k+1})\mid\mathcal F_k]\le H(\pi_k).
 \label{eq:self-distillation-entropy}
\end{equation}
The inequality is strict under nonzero conditional update variance.
\end{theorem}
\noindent\hyperref[app:variance-accumulation]{Proof: Appendix~\ref*{app:variance-accumulation}.}

Thus sampling noise can deterministicly reduce expected entropy even when the parameter update has zero mean, while iterative generation severely amplifies this effect. We should note that the iterative training is unavoidable in RL or self-distillation theoretically, since the distribution drift already occurs, the objective itself will require updating the actor's parameter or introducing importance sampling correction~\citep{schulman_proximal_2017}.

This expectational decrease of entropy should be understood in this way: We have no guarantee on the deterministic direction of entropy in each step. If one training step happens to generate some samples with low probability, self-distillation on these samples will actually increase the sample entropy. So this is not contradictive with previous researches that carefully analyzed the condition when a single learning step will increase or decrease the entropy~\citep{cui_entropy_2025, ren_learning_2025, jin_revisiting_2026, wang_entropy_2026}. This result in fact takes a step further, and answers if we are doing the update multiple times, or in the long-term of training, how should the entropy change on average. The long-term expectation cancels the randomness in a single step update, and gives a deterministic decrement.

\subsection{Reversed Iterative Self-Distillation}

A common explanation for entropy collapse in large-model training is that high-probability samples are more likely to be sampled, and further increasing their probabilities leads to an additional reduction in entropy~\citep{cui_entropy_2025, hao_rethinking_2026}. This interpretation is consistent with our findings. In fact, the expected entropy decrease established by Theorem~\ref{eq:self-distillation-entropy} further shows that, although low-probability samples may occasionally be sampled and increase the entropy during training, such increases are, in expectation, dominated by the overall decreasing trend.

To further illustrate this phenomenon, we additionally conduct a reversed iterative self-distillation experiment. In this setting, we assign a reward of $-1$ to all samples. As a result, high-probability samples are also more likely to be penalized, thereby tending to increase the entropy. Nevertheless, after reversing the sign of the objective, it still retains the properties of zero mean and nonzero variance, and the expected entropy decrease predicted by Theorem~\ref{eq:self-distillation-entropy} continues to hold. This is a relatively counterintuitive result. It further predicts that the rare entropy-decreasing events under such training should be particularly pronounced, which is indeed what we observe empirically, as shown in Figure~\ref{fig:distillation-comp}. More results can be found in Section~\ref{app:supp-exp-isd}.

\begin{figure}[!htb]
  \centering
  \includegraphics[width=\linewidth]{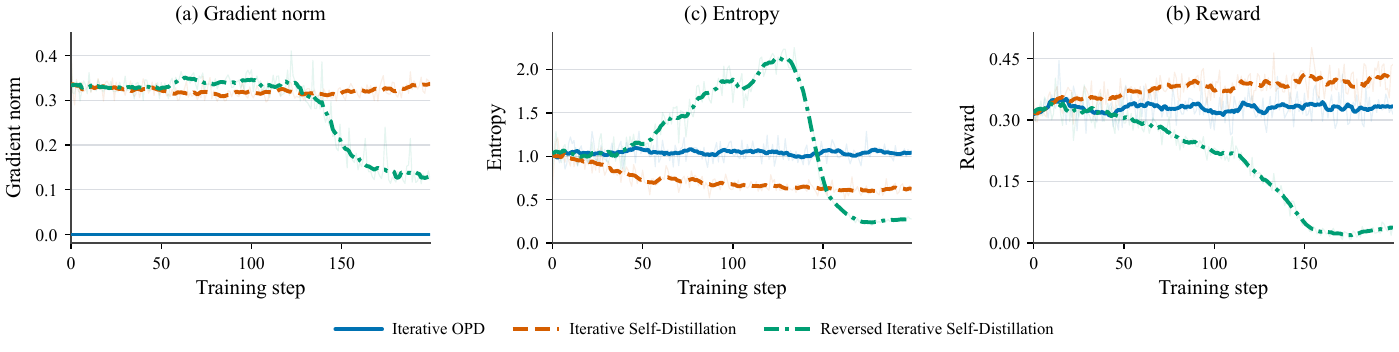}
  \caption{The comparision of the entropy dynamics of distillation algorithms. The nonzero variance and iteration are two necessasy conditions for variance accumulation to occur. RL is where both conditions are satisfied.}
  \label{fig:distillation-comp}
\end{figure}

%
%
%


\subsection{Variance Accumulation in RL}

While the entropy collapse is widely observed in RL scenarios, previous experiments could provide certain evidence that the variance accumulation is also one of the main causes of this phenomena in RL. Another important side evidence of the variance impact in RL is the advantage estimation methods adopted in PPO~\citep{schulman_proximal_2017}, GRPO~\citep{shao_deepseekmath_2024}, or RLOO~\citep{ahmadian_back_2024}. These methods can effectively lower down the variance of the estimator in RL algorithms. This notably increases the stability of RL training, while our deduction further indicates that they are can also prevent potential entropy collapse of the target distribution. It remains open if we could find better methods in controling the estimattion variance.

\begin{figure}[!htb]
  \centering
  \includegraphics[width=\linewidth]{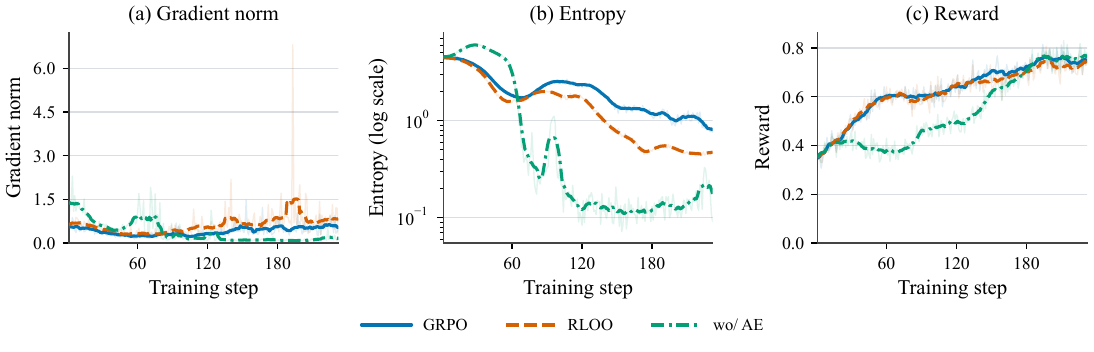}
  \caption{Ablation study on the impact of advantage estimation methods. We can see that under the same RL setting, advantage estimation methods constantly lowers the entropy deduction. Although they eventually achieved similar reward on the training set, to some extent, lower entropy still reflects reduced diversity and weaker generalization ability.}
  \label{fig:ablation-ae-rl}
\end{figure}
%

\section{Conditional KL Regularization in RL}

\begin{insight}
  In iterative algorithms, the target distribution helps separate intended gains from unintended side effects. For example, RL aims to maximize expected reward. With binary rewards, conditioning on $r(x,y)=1$ yields the training target $\pi_\theta^+$, which remains unchanged under the ideal update in Theorem~\ref{lem:reward-reweighting}. Tracking how this target changes during training can therefore reveal side effects beyond the intended reward improvement.
\end{insight}

\subsection{Conditional Decomposition in RL}

A sequence-level binary reward naturally decomposes the probability distribution induced by the original model into two components, with the true conditional distribution serving as the target distribution. The policy entropy also decomposes into the binary entropy of correctness and an accuracy-weighted average of the true and false conditional entropies.

\begin{lemma}[Accuracy-weighted entropy decomposition]
\label{lem:rl-target-entropy}
For a prompt $x$ with deterministic binary rewards, write $\rho=\rho_\theta(x)\in(0,1)$ for the accuracy and suppress conditioning on $x$. Let $H(\pi_\theta^+)$ and $H(\pi_\theta^-)$ denote the true conditional and false conditional entropies, respectively. Then
\begin{equation}
 H(\pi_\theta)=h_{\mathrm b}(\rho)
    +\rho H(\pi_\theta^+)+(1-\rho)H(\pi_\theta^-),
 \label{eq:rl-entropy-decomposition}
\end{equation}
where $h_{\mathrm b}(\rho)=-\rho\log\rho-(1-\rho)\log(1-\rho)$.
If $H(\pi_\theta^+)<H(\pi_\theta^-)$, increasing accuracy while keeping both conditional distributions fixed will decrease their weighted average entropy.
\end{lemma}
\noindent\hyperref[app:rl-target-entropy]{Proof: Appendix~\ref*{app:rl-target-entropy}.}

Typically in practice, both conditional entropies are much larger than the binary entropy $h_{\mathrm b}(\rho)\le\log 2$, the total entropy is well approximated by their accuracy-weighted average. This conditional is supported in our experiments, see Figure~\ref{fig:kl-ckl-comparison}. Both conditional entropies are significantly larger than the binary entropy, while false conditional entropy is usually several times higher than true conditional entropy.

This property can, to some extent, account for the approximately linear relationship between entropy and accuracy observed in prior studies~\citep{cui_entropy_2025}. However, a mere shift in probability density is insufficient to explain the reductions in diversity and pass@k performance reported in other studies~\citep{yue_does_2025, tajwar_maximum_2026}. This accuracy-driven source of entropy reduction in RL need not reflect a loss of diversity among correct responses. The concern with diversity collapse is the degeneration of the target and conditional entropy.


This conditional decomposition induced by RL also offers a new perspective on regularization. To prevent model collapse and preserve generalization, a common practice is to introduce a KL-divergence constraint that keeps the trained model from deviating excessively from the original model~\citep{schulman_trust_2017, schulman_proximal_2017, ouyang_training_2022}. However, directly constraining the KL divergence can also limit improvements in reward. Under the conditional decomposition of RL, what we truly need to control is the degradation of the target distribution. This observation motivates a Conditional KL (CKL) regularization scheme that directly constrains changes in the target distribution.

This CKL regularization admits a simple policy-gradient estimator:

\begin{insight}
The mean log probability-ratio between the current and reference models over positive samples captures the average probability gain on correct solutions. CKL then regularizes only the deviation of each positive sample from this average gain. This actually introduces a baseline into the KL penalty for positive samples, while imposing no constraint on negative samples.
\end{insight}

\begin{theorem}[Centered sequence-$k_1$ estimator for CKL]
\label{thm:ckl-estimator}
Fix a prompt $x$ and suppress its conditioning. Let $\pi_\theta$ be a differentiable positive policy and $\pi_{\mathrm{ref}}$ a fixed reference. Define
$K(\theta)=D_{\mathrm{KL}}(\pi_\theta^+\Vert\pi_{\mathrm{ref}}^+)$.
Draw $G$ independent responses $Y_i\sim\pi_{\theta_0}$, and let
$I^+=\{i:r(x,Y_i)=1\}$ and $N=|I^+|$. The empirical signal is
\begin{align}
  d_i=\log\frac{\pi_{\theta_0}(Y_i)}{\pi_{\mathrm{ref}}(Y_i)}
  ,\notag\qquad
  \widehat c_i =
 \begin{cases}
 d_i-\dfrac{1}{N}\sum_{j\in I^+}d_j,&i\in I^+,\ N\ge2,\\
 0,&\text{otherwise}.
 \end{cases}
 \label{eq:ckl-empirical-signal}
\end{align}
Writing $s_i=\left.\nabla_\theta\log\pi_\theta(Y_i)\right|_{\theta_0}$,
the estimator satisfies, for every $n$ with $\Pr(N=n)>0$,
\begin{equation}
 \mathbb E\!\left[\left.\frac1G\sum_{i=1}^G\widehat c_i s_i\,\right|N=n\right]
 =\frac{\max(n-1,0)}{G}\left.\nabla_\theta K(\theta)\right|_{\theta_0}.
 \label{eq:ckl-finite-gradient}
\end{equation}
Thus a detached reward $-\beta\widehat c_i$, with $\beta\ge0$, gives an expected policy-gradient contribution proportional to the negative conditional KL gradient.
\end{theorem}
\noindent\hyperref[app:ckl-estimator]{Proof: Appendix~\ref*{app:ckl-estimator}.}

\begin{figure}[htb]
  \centering
  \includegraphics[width=\linewidth]{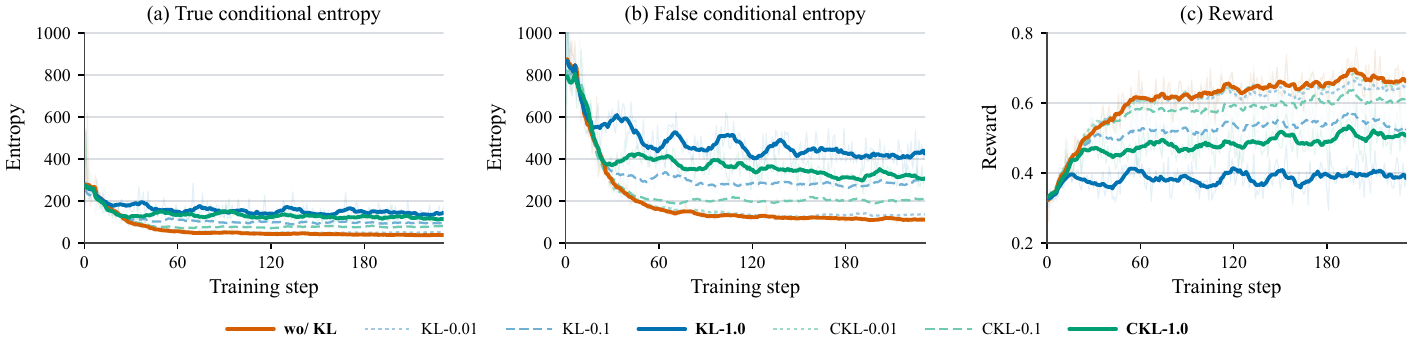}
  \caption{The comparison of KL and Conditional KL (CKL) regularization in RL. We trained Qwen3-1.7B-Base~\citep{yang_qwen3_2025} on MATH~\citep{hendrycks_measuring_2021} dataset with GRPO~\citep{shao_deepseekmath_2024} algorithm and different KL penalities. Two conditional entropies are measured on the sequence level. We can see that strong KL control strength severely limits the reward increment in RL, while CKL regularization allows the model to obtain higher reward while maintaining the diversity of true samples.}
  \label{fig:kl-ckl-comparison}
\end{figure}

\begin{figure}[!htb]
  \centering
  \includegraphics[width=\linewidth]{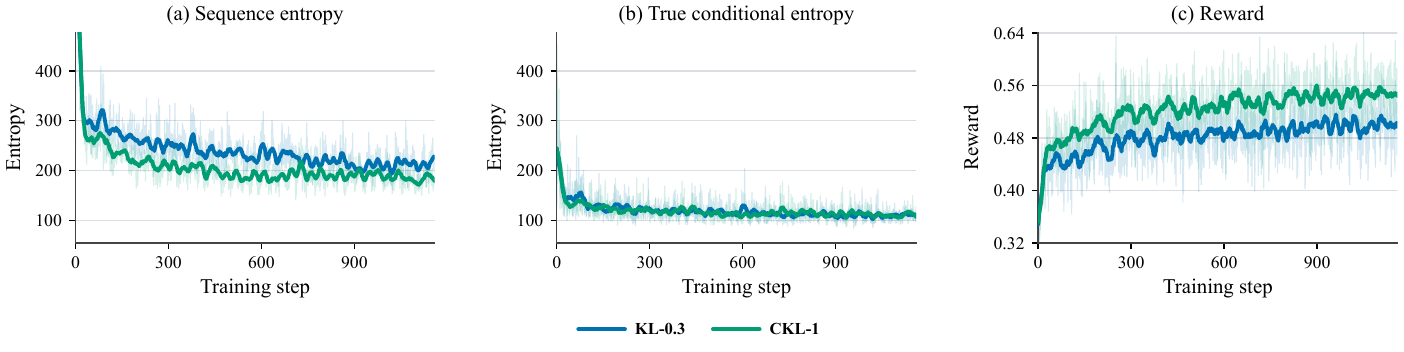}
  \caption{Extended experiments on KL and CKL regularization. These experiments follows the same setting as that in Figure~\ref{fig:kl-ckl-comparison}. We can see that CKL-1 and KL-0.3 exhibits similar regularization ability on conditional distribution, while CKL-1 constantly obtains higher accuracy. Its entropy will thus be slightly lower according to the accuracy-weighted decomposition.}
  \label{fig:ext-ckl-cmp}
\end{figure}

As shown in Figure~\ref{fig:kl-ckl-comparison}, we indeed observe that CKL effectively relaxes the constraint imposed by standard KL regularization on accuracy improvements while preserving the true conditional entropy. Although CKL cannot achieve completely lossless regularization due to cross-sequence interactions induced by model parameterization, it nevertheless demonstrates substantial practical value.


This advantage of CKL still holds for up to 20 training epochs, as shown in Figure~\ref{fig:ext-ckl-cmp}. We then evaluated these models on different math datasets, where CKL regularization did exhibit higher performance and generalizability over KL of similar strength, See Figure~\ref{fig:passk-eval}. More related results could be found in Section~\ref{app:rl-exps}.

\begin{figure}[!htb]
  \centering
  \includegraphics[width=\linewidth]{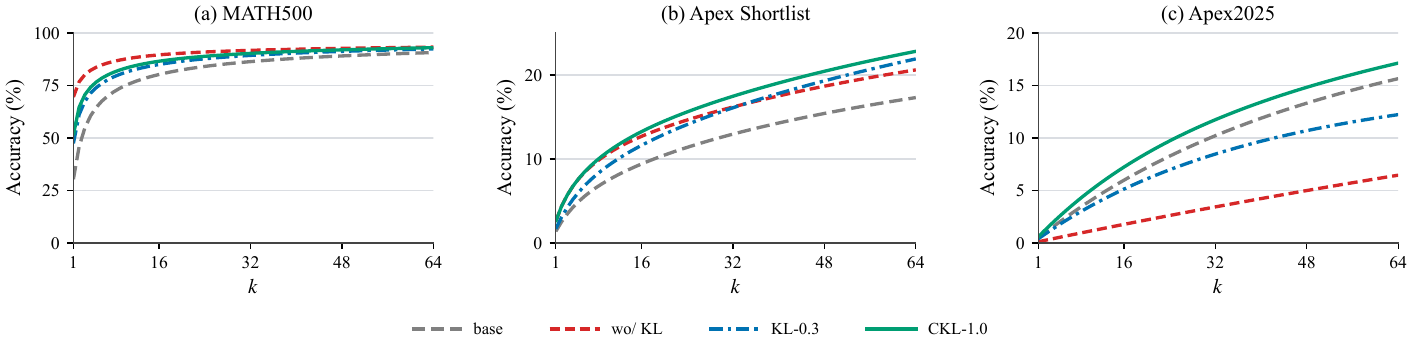}
  \caption{The comparison of pass@k performance on MATH500~\citep{lightman_lets_2023}, Apex Shortlist~\citep{balunovic_matharena_2025} and Apex2025~\citep{balunovic_matharena_2025}. Models are trained for 20 epochs on MATH dataset. We can see that models trained without KL has higher accuracy on datasets similar to the training set, while severe collapse occurs on harder datasets like Apex2025. CKL constantly achieves higher generalizability over KL regularization across these experiments.}
  \label{fig:passk-eval}
\end{figure}

\section{Conclusions}

Distribution matching offers a new perspective on LLM training algorithms. The objective of a variety of LLM training algorithms can be unified as a divergence between the policy model and a target distribution. The target distribution itself also serves as an effective lens through which the properties of training algorithms can be analyzed.

In this work, we show that forward KL distribution matching exhibits non-vanishing variance in the proximal limit, which provides a possible explanation for the advantage of OPD over SFT in proximal settings. Furthermore, for iterative algorithms such as RL, we find that the cumulative effect of this non-vanishing variance can lead to entropy collapse of the target distribution. This observation offers a potential explanation for the degradation of output diversity observed under prolonged RL training in related studies. Motivated by the role of the target distribution in RL, we further develop a Conditional KL regularization method and empirically demonstrate its advantages over classical KL regularization in language model. We hope that our findings can provide useful insights for future research and practical applications in this area.

%
%

\subsection*{AI use statement}

In this work, we used generative AI tools for figure illustration and polishment of the writing. We also used them to retrieve and summarize related works and modifying the training code. AI tools also helped us in early theoretical exploration. The ideation and evaluation were done by the authors. In paper writing we used AI tool to summarize and refine our drafts in theoretical exploration. We did not use generative AI tools for directly draft sections of the paper or generating synthetic datasets.

\bibliography{ce}
\bibliographystyle{iclr2027_conference}

\clearpage
\appendix
\numberwithin{equation}{section}
\section{Related Works}

Modern LLM training typically begins with pretraining on massive datasets, with the goal of modeling the natural language distribution and acquiring broad world knowledge. Mid-training generally follows the same underlying learning paradigm, while differences in training data and parameter configurations can lead to substantially different outcomes. In the post-training stage, a common practice is to train specialized expert models for tasks in different domains, typically using SFT for cold start followed by task-specific RL algorithms. These expert models can then be distilled into a unified model through methods such as OPD, resulting in a final general-purpose model~\citep{deepseek-ai_deepseek-v4_2026, team_kimi_2026, glm-5-team_glm-5_2026}.

\subsection{Distribution Matching Perspective}

Among the algorithms used across different stages of LLM training, most admit a relatively clear interpretation from the perspective of distribution matching, with RL being a notable exception. Despite its distinctive optimization formulation, RL has also been widely recognized as closely related to distribution matching~\citep{korbak_reinforcement_2022}. Prior to our work, a number of studies investigated the connection between RL and distribution-matching objectives and derived locally optimal solutions with an exponential-tilting form similar to that in our Theorem~\ref{lem:reward-reweighting}~\citep{korbak_rl_2022, korbak_reinforcement_2022}. Reward-induced conditional distributions have likewise been studied in prior work and have motivated several algorithmic improvements~\citep{tajwar_maximum_2026, zhou_variational_2025}.

In most previous studies, however, such distributions primarily appeared as auxiliary results in the derivation or analysis of RL objectives. Subtle discrepancies remain between RL optimization and direct distribution matching, preventing the two perspectives from being fully unified~\citep{korbak_reinforcement_2022}. Recent studies on MaxRL provide an important step toward bridging this gap, revealing a first-order approximation between RL and distribution matching. More importantly, our work differs in that we treat the target distribution induced by distribution matching as a central object of study, rather than merely an intermediate analytical construct. In particular, we study both its role in estimating the training objective and its stochastic evolution throughout training, which together constitute the main results of this work. We expect that this perspective may provide a useful lens for further understanding and improving related training algorithms.

\subsection{Entropy Collapse in RL}

Diversity collapse and entropy decrement in reinforcement learning have long been recognized in the classical RL literature. RL optimization can drive a model toward a suboptimal solution with reduced behavioral diversity, thereby limiting the model's long-term training potential. Classical approaches to mitigating this issue include entropy regularization~\citep{ahmed_understanding_2019} and constraints based on KL divergence~\citep{schulman_trust_2017}. These methods, however, introduce their own trade-offs, including optimization instability and restrictions on reward improvement.

As RL has become increasingly prevalent in LLM post-training, a growing body of work has revisited entropy collapse in this setting. Existing studies have investigated its causes and potential remedies from a variety of perspectives, including biases introduced by PPO clipping~\citep{yu_dapo_2025, cai_beyond_2026, jin_revisiting_2026}, geometric biases induced by model architecture~\citep{ren_learning_2025}, and the underlying training dynamics~\citep{wang_entropy_2026}. In contrast, our analysis identifies a long-horizon degradation of the target distribution caused by variance accumulation. This phenomenon arises at a more fundamental level and does not depend on a particular RL algorithm or implementation trick. Based on this analysis, we further derive a effective regularization strategy. Entropy collapse may represent one of the central challenges in scaling post-training, and our results provide an additional perspective on its underlying mechanism.

\subsection{Variance Control and Training Stability}

The stability of RL algorithms has been studied extensively, with the high variance of training-objective estimators being one of the primary challenges. A substantial fraction of algorithmic developments in RL can be understood as attempts to address this issue~\citep{konda_onactor-critic_2003, schulman_trust_2017, schulman_proximal_2017}. Advantage estimation is among the most widely adopted approaches: it can be implemented using a separate critic model~\citep{konda_onactor-critic_2003, schulman_trust_2017, schulman_proximal_2017}, or approximated through techniques such as group-based sampling and normalization~\citep{shao_deepseekmath_2024, ahmadian_back_2024}. Beyond advantage estimation, the treatment of importance weights and the design of the underlying training infrastructure can also have a substantial impact on optimization stability~\citep{zheng_group_2025, gao_soft_2025}.

Recent empirical evidence suggests that, in the LLM setting, the strong prior induced by extensive pretraining can enable stable RL training over a meaningful range even in the absence of aggressive variance-reduction techniques~\citep{ahmadian_back_2024}. Nevertheless, controlling training variance can still yield measurable performance gains. Our analysis further provides a possible explanation for these gains: reducing training variance not only stabilizes individual optimization steps, but can also mitigate the degradation of the training target caused by variance accumulation over time. As a result, the model is more likely to preserve meaningful exploration and reach better regions of the parameter space.

There are also many other works studying the stability issue and improvements of OPD~\citep{li_rethinking_2026, oh_kl_2026}, while a most simple yet effective method could be a two-stage distillation, which uses SFT as coldstart, and use OPD to further improve their performance~\citep{yang_qwen3_2025}. Our analysis could provide more theoretical insights on the distillation algorithms as well.

\section{Limitation}

In the previous sections we proposed a framework that unifies different training algorithms via distribution matching objective. For RL algorithms we only analyzed the simplest 0-1 binary reward setting. For more generalized multi-value reward RL we can also extend the distribution matching by introducing a randomized acceptance event, see Appendix~\ref{app:conditional-kl-maxrl}. However, a detailed investigation of its properties is left for future work.

A more important limitation of this work is that, we did not consider generalization ability of these methods. The main theoretical results derived in this paper apply directly to the behavior on the training set, while a more complete understanding of the corresponding behavior on the test set requires further theoretical analysis. Nevertheless, many of our experiments are also conducted on the test set, where we empirically observe behavior similar to that on the training set.


\section{Proof of Reward Reweighting Under a Forward-KL Restriction}
\label{app:reward-reweighting}

\begin{insight}
The forward-KL penalty makes the optimum interior and unique, so a single
Lagrange multiplier determines how each response probability is reweighted.
Responses with the same reward receive the same multiplier, which cancels
when conditioning on a reward level set.
\end{insight}

Fix a prompt $x$ and a finite response space. Write
$p_y=\pi_\theta(y\mid x)>0$, $r_y=r(x,y)\in\mathbb R$, and $\beta>0$.

\begin{proof}[Proof of Theorem~\ref{lem:reward-reweighting}]
On the probability simplex,
\begin{equation}
 F(q)=\sum_yq_yr_y-\beta D_{\mathrm{KL}}(p\Vert q)
 \label{eq:app-reward-forward-objective}
\end{equation}
has Hessian $-\beta\operatorname{diag}(p_y/q_y^2)\prec0$ and tends
to $-\infty$ at the boundary. Its unique maximizer is interior, with
Lagrange conditions
\begin{equation}
 r_y+\frac{\beta p_y}{q_y}-\lambda_x=0
 \quad\Longrightarrow\quad
 q_y^*=\frac{\beta p_y}{\lambda_x-r_y},\qquad
 \sum_y\frac{\beta p_y}{\lambda_x-r_y}=1.
 \label{eq:app-reward-forward-stationarity}
\end{equation}
The normalization sum decreases continuously from $+\infty$ to $0$
on $(\max_y r_y,\infty)$, determining a unique $\lambda_x$.
For any nonempty $\mathcal S_a=\{y:r_y=a\}$,
\begin{equation}
 q^*(y\mid\mathcal S_a)
 =\frac{p_y\beta/(\lambda_x-a)}
        {\sum_{z\in\mathcal S_a}p_z\beta/(\lambda_x-a)}
 =p(y\mid\mathcal S_a),\qquad y\in\mathcal S_a.
 \label{eq:app-reward-conditional-invariance}
\end{equation}
\end{proof}

\paragraph{Binary rewards.}
For $r_y\in\{0,1\}$ and $0<\rho_\theta(x)<1$, the same solution is
\begin{equation}
 q_y^*=\frac{p_y e^{\alpha_xr_y}}
 {1-\rho_\theta(x)+\rho_\theta(x)e^{\alpha_x}},
 \qquad \alpha_x=\log\frac{\lambda_x}{\lambda_x-1}>0.
 \label{eq:app-reward-binary-tilt}
\end{equation}
Thus $q^*(\mathcal S_1)>\rho_\theta(x)$ and $(q^*)^+=\pi_\theta^+$.
For constant rewards, $q^*=p$.

\paragraph{KL constraint.}
Set $\delta_\beta=D_{\mathrm{KL}}(p\Vert q^*)$.
For any $q$ with $D_{\mathrm{KL}}(p\Vert q)\le\delta_\beta$, optimality gives
\begin{equation}
 \mathbb E_q[r]\le\mathbb E_{q^*}[r]
 +\beta\bigl(D_{\mathrm{KL}}(p\Vert q)-\delta_\beta\bigr)
 \le\mathbb E_{q^*}[r],
 \label{eq:app-reward-constrained}
\end{equation}
with equality only at $q=q^*$.

\section{Proof of the True Conditional KL Identity and Its MaxRL Connection}
\label{app:conditional-kl-maxrl}

\begin{insight}
Conditioning on success rescales every successful response by the same
factor, so the KL reduces to the negative logarithm of the success probability.
Differentiating this probability and using the score identity then gives
the success-weighted MaxRL gradient.
\end{insight}

Let $\mathcal D$ be a fixed empirical prompt distribution. For each $x$,
assume a finite response space, a differentiable positive policy
$\pi_\theta$, and a fixed binary reward. Write
$\mathcal S_x^+=\{y:r(x,y)=1\}$ and assume $\rho_\theta(x)>0$.

\begin{proof}[Proof of Theorem~\ref{lem:conditional-kl-maxrl}]
Since $\pi_\theta^+(y\mid x)=\pi_\theta(y\mid x)/\rho_\theta(x)$ on
$\mathcal S_x^+$ and vanishes elsewhere,
\begin{equation}
 D_{\mathrm{KL}}(\pi_\theta^+\Vert\pi_\theta)
 =\sum_{y\in\mathcal S_x^+}\pi_\theta^+(y\mid x)
       \log\frac1{\rho_\theta(x)}
 =-\log\rho_\theta(x).
 \label{eq:app-success-kl}
\end{equation}
Taking the negative expectation over $x$ proves
Eq.~\eqref{eq:conditional-kl-maxrl}. Moreover,
\begin{equation}
 \nabla_\theta\log\rho_\theta(x)
 =\frac{\sum_y r(x,y)\nabla_\theta\pi_\theta(y\mid x)}{\rho_\theta(x)}
 =\mathbb E_{y\sim\pi_\theta^+(\cdot\mid x)}
       [\nabla_\theta\log\pi_\theta(y\mid x)],
 \label{eq:app-success-score}
\end{equation}
and hence
\begin{equation}
 \nabla_\theta\mathcal J_{\mathrm{MaxRL}}^{(\infty)}
 =\mathbb E_{x\sim\mathcal D,y\sim\pi_\theta(\cdot\mid x)}
 \left[\frac{r(x,y)}{\rho_\theta(x)}
              \nabla_\theta\log\pi_\theta(y\mid x)\right].
 \label{eq:maxrl-gradient}
\end{equation}
\end{proof}

\paragraph{Frozen target.}
For $q_x=\pi_{\theta_0}^+(\cdot\mid x)$,
\begin{equation}
 D_{\mathrm{KL}}(q_x\Vert\pi_\theta)
 =D_{\mathrm{KL}}(q_x\Vert\pi_\theta^+)-\log\rho_\theta(x).
 \label{eq:app-success-frozen-decomposition}
\end{equation}
The first term has zero gradient at $\theta_0$, so
\begin{equation}
 \left.-\nabla_\theta\mathbb E_{x\sim\mathcal D}
 D_{\mathrm{KL}}(q_x\Vert\pi_\theta)\right|_{\theta_0}
 =\left.\nabla_\theta\mathcal J_{\mathrm{MaxRL}}^{(\infty)}
 \right|_{\theta_0}.
 \label{eq:app-success-frozen-gradient}
\end{equation}

\paragraph{Finite-order MaxRL.}
Using $\log\rho=-\sum_{k\ge1}(1-\rho)^k/k$ for $0<\rho\le1$, the
order-$K$ objective of \citet{tajwar_maximum_2026} satisfies
\begin{align}
 \mathcal J_{\mathrm{MaxRL}}^{(K)}
 &=-\mathbb E_{x\sim\mathcal D}
       \left[\sum_{k=1}^{K}\frac{(1-\rho_\theta(x))^k}{k}\right],
 \label{eq:app-maxrl-truncation}\\
 \nabla_\theta\mathcal J_{\mathrm{MaxRL}}^{(K)}
 &=\mathbb E_{x\sim\mathcal D}
       \left[\frac{1-(1-\rho_\theta(x))^K}{\rho_\theta(x)}
                         \nabla_\theta\rho_\theta(x)\right].
 \label{eq:app-maxrl-truncated-gradient}
\end{align}
The geometric sum gives the gradient, and $1/k\le1/(K+1)$ for $k>K$ gives
\begin{equation}
 0\le\mathcal J_{\mathrm{MaxRL}}^{(K)}
             -\mathcal J_{\mathrm{MaxRL}}^{(\infty)}
 \le\mathbb E_{x\sim\mathcal D}
       \left[\frac{(1-\rho_\theta(x))^{K+1}}
                    {(K+1)\rho_\theta(x)}\right].
 \label{eq:app-maxrl-truncation-error}
\end{equation}
Thus $\mathcal J_{\mathrm{MaxRL}}^{(1)}=\mathcal J_{\mathrm{RL}}-1$.
Finite prompt support and $\rho_\theta(x)>0$ imply convergence of the
objective and gradient to Eqs.~\eqref{eq:conditional-kl-maxrl}
and~\eqref{eq:maxrl-gradient} as $K\to\infty$.

\subsection{Extension to Nonbinary Rewards via Randomized Acceptance}
\label{app:randomized-reward-extension}

\begin{insight}
Rescaling a bounded reward to $[0,1]$ allows it to serve as an acceptance
probability. Conditioning on acceptance extends the MaxRL identity to
the joint response--acceptance distribution; standard RL is recovered
at first order.
\end{insight}

Under the same policy assumptions, let the fixed reward satisfy
$a\le r(x,y)\le b$, with constants $a<b$, and set
$s(x,y)=(r(x,y)-a)/(b-a)$.
Introduce $A\mid x,Y\sim\operatorname{Bernoulli}(s(x,Y))$, with kernel
$K_s(1\mid x,y)=s(x,y)$, $K_s(0\mid x,y)=1-s(x,y)$. Define
\begin{equation}
 P_\theta(y,A\mid x)=\pi_\theta(y\mid x)K_s(A\mid x,y),
 \qquad
 \rho_\theta^s(x)=P_\theta(A=1\mid x)
 =\mathbb E_{\pi_\theta(\cdot\mid x)}[s(x,Y)].
 \label{eq:app-randomized-joint}
\end{equation}
Assume $\rho_\theta^s(x)>0$ and condition on $A=1$:
\begin{equation}
 Q_\theta(y,A\mid x)
 =P_\theta(y,A\mid x,A=1)
 =\frac{\pi_\theta(y\mid x)s(x,y)}{\rho_\theta^s(x)}
   \mathbf1\{A=1\}.
 \label{eq:app-randomized-target}
\end{equation}
Since $Q_\theta/P_\theta=1/\rho_\theta^s(x)$ on the support of $Q_\theta$,
\begin{equation}
 -\mathbb E_{x\sim\mathcal D}
 D_{\mathrm{KL}}\bigl(Q_\theta(\cdot,\cdot\mid x)
                 \Vert P_\theta(\cdot,\cdot\mid x)\bigr)
 =\mathbb E_{x\sim\mathcal D}
   \left[\log\mathbb E_{\pi_\theta(\cdot\mid x)}[s(x,Y)]\right].
 \label{eq:app-randomized-kl}
\end{equation}
For the response marginal
$q_\theta^s(y\mid x)=\pi_\theta(y\mid x)s(x,y)/\rho_\theta^s(x)$,
the KL chain rule gives
\begin{equation}
 D_{\mathrm{KL}}(Q_\theta\Vert P_\theta)
 =D_{\mathrm{KL}}(q_\theta^s\Vert\pi_\theta)
  +\mathbb E_{q_\theta^s}[-\log s(x,Y)].
 \label{eq:app-randomized-marginal-kl}
\end{equation}
The expectation is over $s(x,y)>0$. For binary $r$ with $s=r$,
the extra term vanishes, recovering Eq.~\eqref{eq:conditional-kl-maxrl}.
Finally, $\log\rho=(\rho-1)+O((1-\rho)^2)$ near $\rho=1$ gives
standard RL on $s$ at first order, and
$\mathbb E[s]=(\mathbb E[r]-a)/(b-a)$.

\begin{table}[!htbp]
\centering
\caption[Distribution matching for bounded nonbinary rewards]{Continuation
of Table~\ref{tab:distribution-matching}: bounded nonbinary rewards.
Objectives are averaged over prompts $x\sim\mathcal D$. The model distribution is $P_\theta$; the target
$Q_\theta$ conditions it on acceptance and has $Q_\theta(y,0\mid x)=0$.
$^{\dagger}$\,RL is its first-order approximation.}
\label{tab:distribution-matching-nonbinary}
\vspace{5pt}
\begingroup
\small
\setlength{\tabcolsep}{5pt}
\renewcommand{\arraystretch}{1.45}
\begin{tabularx}{\linewidth}{
 >{\raggedright\arraybackslash}p{0.13\linewidth}
 >{\centering\arraybackslash}X
 >{\centering\arraybackslash}p{0.235\linewidth}
 >{\raggedright\arraybackslash}p{0.16\linewidth}}
\specialrule{\heavyrulewidth}{0pt}{0pt}
\rowcolor{frameworktint}
\textcolor{frameworkink}{\textbf{Algorithm}} &
\textcolor{frameworkink}{\textbf{Objective function}} &
\textcolor{frameworkink}{\textbf{Target distribution}} &
\textcolor{frameworkink}{\textbf{Divergence}} \\
\algorithmgroup{Iterative algorithms}
RL &
$\mathbb E_{\pi_\theta}[s(x,Y)]=\rho_\theta^s(x)$ &
$\begin{gathered}
 Q_\theta(y,1\mid x)\\
 =\dfrac{s(x,y)\targetpolicy(y\mid x)}{\rho_\theta^s(x)}
\end{gathered}$ &
\forwardKL$^{\dagger}$ \\
\addlinespace[4pt]
MaxRL (extended) &
$\begin{gathered}
 \log\mathbb E_{\pi_\theta}[s(x,Y)]\\
 =-D_{\mathrm{KL}}(Q_\theta\Vert P_\theta)
\end{gathered}$ &
$\begin{gathered}
 Q_\theta(y,1\mid x)\\
 =\dfrac{s(x,y)\targetpolicy(y\mid x)}{\rho_\theta^s(x)}
\end{gathered}$ &
\forwardKL \\
\bottomrule
\end{tabularx}
\endgroup
\end{table}
\FloatBarrier

\section{Proof of Gradient Noise at the Proximal Limit}
\label{app:proximal-gradient-noise}

\begin{insight}
When student and teacher coincide, every OPD log-ratio weight is zero,
whereas the sampled KD scores cancel only in expectation.
\end{insight}

Fix a prompt $x$, a finite response space, and a fixed positive teacher
$\pi_t$. Let $\pi_\theta>0$ be continuously differentiable near
$\theta_0$, with $\pi_{\theta_0}=\pi_t$.
Write $s_\theta(y\mid x)=\nabla_\theta\log\pi_\theta(y\mid x)$ and
$w_\theta(y\mid x)=\log[\pi_t(y\mid x)/\pi_\theta(y\mid x)]$.
For $m\ge1$, define the sampled-response estimators
\begin{align}
 \widehat g_{\mathrm{KD}}(\theta)
 &=\frac1m\sum_{i=1}^m s_\theta(Y_i\mid x),
 &Y_i&\overset{\mathrm{iid}}{\sim}\pi_t(\cdot\mid x),
 \label{eq:proximal-kd-estimator}\\
 \widehat g_{\mathrm{OPD}}(\theta)
 &=\frac1m\sum_{i=1}^m w_\theta(\widetilde Y_i\mid x)
                          s_\theta(\widetilde Y_i\mid x),
 &\widetilde Y_i&\overset{\mathrm{iid}}{\sim}\pi_\theta(\cdot\mid x).
 \label{eq:proximal-opd-estimator}
\end{align}

\begin{proof}[Proof of Lemma~\ref{lem:proximal-gradient-noise}]
Since $\mathbb E_{\pi_\theta}[s_\theta]=\nabla_\theta\sum_y\pi_\theta(y\mid x)=0$,
\begin{align}
 \nabla_\theta[-D_{\mathrm{KL}}(\pi_t\Vert\pi_\theta)]
 &=\mathbb E_{\pi_t}[s_\theta],
 \label{eq:proximal-kd-gradient}\\
 \nabla_\theta[-D_{\mathrm{KL}}(\pi_\theta\Vert\pi_t)]
 &=\mathbb E_{\pi_\theta}[(w_\theta-1)s_\theta]
 =\mathbb E_{\pi_\theta}[w_\theta s_\theta],
 \label{eq:proximal-opd-gradient}
\end{align}
so both estimators are unbiased. At $\theta_0$, $w_{\theta_0}=0$ and
$\mathbb E_{\pi_t}[s_{\theta_0}]=0$. Therefore
$\widehat g_{\mathrm{OPD}}(\theta_0)=0$ almost surely, while independence gives
\begin{equation}
 \operatorname{Cov}(\widehat g_{\mathrm{KD}}(\theta_0))
 =\frac1m\mathbb E_{\pi_t}[s_{\theta_0}s_{\theta_0}^{\top}]
 =\frac1m F_x(\theta_0).
 \label{eq:app-proximal-fisher}
\end{equation}
For $A\in\{\mathrm{KD},\mathrm{OPD}\}$, set
$V_A(\theta)=\mathbb E\|\widehat g_A(\theta)
-\mathbb E\widehat g_A(\theta)\|_2^2$.
Then
\begin{align}
 mV_{\mathrm{KD}}(\theta)
 &=\mathbb E_{\pi_t}\|s_\theta\|_2^2
       -\|\mathbb E_{\pi_t}s_\theta\|_2^2,\nonumber\\
 mV_{\mathrm{OPD}}(\theta)
 &=\mathbb E_{\pi_\theta}[w_\theta^2\|s_\theta\|_2^2]
       -\|\mathbb E_{\pi_\theta}[w_\theta s_\theta]\|_2^2.
 \label{eq:proximal-estimation-error}
\end{align}
These finite sums are continuous in $\theta$. If
$\operatorname{Tr}F_x(\theta_0)>0$, then
$V_{\mathrm{OPD}}(\theta_0)=0<V_{\mathrm{KD}}(\theta_0)$, so the strict
inequality persists in a neighborhood of $\theta_0$.
\end{proof}

\section{A Conditional Variance Comparison in the Distal Regime}
\label{app:distal-gradient-variance}

\begin{insight}
Suppressing teacher probability on a set still sampled by the student
introduces a growing $\log(1/\epsilon)$ term in the OPD weights.
Separating this term from a bounded remainder reveals quadratic growth
of OPD variance under the stated nonzero-score condition, while KD variance
remains bounded by the fixed student scores.
\end{insight}

Fix $x$, a finite response space $\mathcal Y_x$, and a positive policy
$\pi_\theta$ continuously differentiable near $\bar\theta$. Write
$p_y=\pi_{\bar\theta}(y\mid x)$ and
$s_y=\nabla_\theta\log\pi_\theta(y\mid x)|_{\bar\theta}$.
For a nonempty proper subset $\mathcal A$ of $\mathcal Y_x$, set
$a=\sum_{y\in\mathcal A}p_y\in(0,1)$. Given fixed positive distributions
$u$ on $\mathcal A$ and $v$ on its complement, define
\begin{equation}
 \pi_{t,\epsilon}(y\mid x)=
 \begin{cases}
   \epsilon u_y,&y\in\mathcal A,\\
   (1-\epsilon)v_y,&y\notin\mathcal A,
 \end{cases}
 \qquad 0<\epsilon<1.
 \label{eq:distal-teacher-family}
\end{equation}
With $m$ independent responses and teacher $\pi_{t,\epsilon}$, let
$V_A(\epsilon)=\operatorname{Tr}\operatorname{Cov}(\widehat g_A(\bar\theta))$
for the estimators in
Eqs.~\eqref{eq:proximal-kd-estimator}--\eqref{eq:proximal-opd-estimator}.

\begin{lemma}[Variance amplification under teacher-probability suppression]
\label{lem:distal-gradient-variance}
If $\sum_{y\in\mathcal A}p_y\|s_y\|_2^2>0$, then, as $\epsilon\to0$,
\begin{equation}
 V_{\mathrm{OPD}}(\epsilon)
 =\frac{C_{\mathcal A}}m\log^2\frac1\epsilon
       +O\!\left(\frac{\log(1/\epsilon)}m\right),
 \qquad V_{\mathrm{KD}}(\epsilon)\le\frac{G^2}{m},
 \label{eq:distal-variance-growth}
\end{equation}
where $G=\max_y\|s_y\|_2$ and
\begin{equation}
 C_{\mathcal A}
 =\sum_{y\in\mathcal A}p_y\|s_y\|_2^2
       -\left\|\sum_{y\in\mathcal A}p_ys_y\right\|_2^2>0.
 \label{eq:distal-variance-coefficient}
\end{equation}
Consequently, for every fixed $R>0$, sufficiently small $\epsilon$
gives $V_{\mathrm{OPD}}(\epsilon)>R V_{\mathrm{KD}}(\epsilon)$.
\end{lemma}

\begin{proof}
Set $L=\log(1/\epsilon)$ and
$w_\epsilon(y)=\log[\pi_{t,\epsilon}(y\mid x)/p_y]$. Then
\begin{equation}
 w_\epsilon(y)=-L\mathbf1_{\mathcal A}(y)+b_\epsilon(y),
 \qquad
 b_\epsilon(y)=
 \begin{cases}
   \log(u_y/p_y),&y\in\mathcal A,\\
   \log(v_y/p_y)+\log(1-\epsilon),&y\notin\mathcal A.
 \end{cases}
 \label{eq:distal-weight-decomposition}
\end{equation}
For $0<\epsilon\le1/2$, $b_\epsilon$ and $s$ are uniformly bounded.
Expanding the variance under $Y\sim p$ gives
\begin{equation}
 mV_{\mathrm{OPD}}(\epsilon)
 =L^2\operatorname{Tr}\operatorname{Cov}_p
       (\mathbf1_{\mathcal A}(Y)s_Y)+O(L)
 =L^2 C_{\mathcal A}+O(L).
 \label{eq:distal-variance-expansion}
\end{equation}
Cauchy--Schwarz yields
$C_{\mathcal A}\ge(1-a)\sum_{y\in\mathcal A}p_y\|s_y\|_2^2>0$.
For KD,
$mV_{\mathrm{KD}}(\epsilon)
\le\mathbb E_{\pi_{t,\epsilon}}\|s_Y\|_2^2\le G^2$,
which proves Eq.~\eqref{eq:distal-variance-growth} and the comparison
for any fixed $R>0$. Also,
$D_{\mathrm{KL}}(p\Vert\pi_{t,\epsilon})=aL+O(1)\to\infty$.
\end{proof}

\section{Martingales and Entropy in Self-Distillation}
\label{app:variance-accumulation}

\begin{insight}
Sampling from the current policy makes the score update have conditional
mean zero, so the parameters form a martingale despite nonzero sampling noise.
\end{insight}

For Lemmas~\ref{lem:self-distillation-martingale}--\ref{lem:self-distillation-entropy}, fix a prompt $x$ and a
finite response space $\mathcal Y_x$. Assume a positive, twice continuously
differentiable policy, integrable iterates and increments, and
$\mathcal F_k$-measurable step sizes $\eta_k$ and batch sizes $m_k\ge1$.
With the teacher frozen within each step, plain SGD gives
\begin{equation}
 \theta_{k+1}=\theta_k+\eta_k\widehat g_k,
 \qquad
 \widehat g_k=\frac1{m_k}\sum_{i=1}^{m_k}s_{\theta_k}(Y_{k,i}\mid x),
 \qquad Y_{k,i}\mid\mathcal F_k
     \overset{\mathrm{iid}}{\sim}\pi_k,
 \label{eq:self-distillation-sgd}
\end{equation}
where $\pi_k(y)=\pi_{\theta_k}(y\mid x)$ and
$s_\theta=\nabla_\theta\log\pi_\theta$.

\begin{proof}[Proof of Lemma~\ref{lem:self-distillation-martingale}]
The score identity and conditional independence yield
\begin{equation}
 \mathbb E[\widehat g_k\mid\mathcal F_k]=0,
 \qquad
 \operatorname{Cov}(\theta_{k+1}-\theta_k\mid\mathcal F_k)
 =\frac{\eta_k^2}{m_k}F_x(\theta_k),
 \label{eq:self-distillation-sgd-covariance}
\end{equation}
so $\mathbb E[\theta_{k+1}\mid\mathcal F_k]=\theta_k$.
If $M_{k,y}$ bounds the operator norm of the probability Hessian along the update segments,
Taylor expansion gives
\begin{equation}
 \left|\mathbb E[\pi_{k+1}(y)\mid\mathcal F_k]-\pi_k(y)\right|
 \le \frac{M_{k,y}\eta_k^2}{2m_k}\operatorname{Tr}F_x(\theta_k)
 =O(\eta_k^2),
 \label{eq:self-distillation-probability-remainder}
\end{equation}
locally as $\eta_k\to0$ for fixed $\theta_k$.
\end{proof}

For the entropy result, let $h_x(\theta)=-\sum_y\pi_\theta(y\mid x)
\log\pi_\theta(y\mid x)$ and $d_k=\theta_{k+1}-\theta_k$.
Assume that, for every possible update $d_k$ and $t\in[0,1]$,
\begin{equation}
 d_k^\top\nabla_\theta^2h_x(\theta_k+t d_k)d_k
 \le-\mu_k\|d_k\|_2^2,
 \qquad \mu_k\ge0,
 \label{eq:self-distillation-entropy-curvature}
\end{equation}
where $\mu_k$ is $\mathcal F_k$-measurable. The case $\mu_k=0$
is concavity along the update segments; $\mu_k>0$ gives strict decrease
whenever the update variance is nonzero.

\begin{insight}
Mean-zero parameter noise can lower expected entropy when entropy is
concave along the possible update segments.
In the Taylor expansion, the linear term vanishes in conditional expectation,
leaving a nonpositive curvature term whose magnitude is bounded using the Fisher covariance.
\end{insight}

\begin{proof}[Proof of Theorem~\ref{lem:self-distillation-entropy}]
Taylor's formula with \eqref{eq:self-distillation-entropy-curvature},
$\mathbb E[d_k\mid\mathcal F_k]=0$, and
\eqref{eq:self-distillation-sgd-covariance} yields
\begin{equation}
 \mathbb E[H(\pi_{k+1})\mid\mathcal F_k]
 \le H(\pi_k)-\frac{\mu_k\eta_k^2}{2m_k}
       \operatorname{Tr}F_x(\theta_k)
 \le H(\pi_k).
 \label{eq:self-distillation-entropy-bound}
\end{equation}
Taking expectations proves the claim, with strict decrease if the
subtracted term is positive with positive probability.
\end{proof}

\section{Proof of the Accuracy-Weighted Entropy Decomposition}
\label{app:rl-target-entropy}

Fix a prompt $x$, a finite response space, and a deterministic binary reward.
Write $\rho=\rho_\theta(x)\in(0,1)$,
$p=\pi_\theta(\cdot\mid x)$, and $p^\pm=\pi_\theta^\pm(\cdot\mid x)$.
Then
\begin{equation}
 p^+(y)=\frac{r(x,y)p(y)}{\rho},
 \qquad
 p^-(y)=\frac{(1-r(x,y))p(y)}{1-\rho}.
 \label{eq:rl-success-failure-distributions}
\end{equation}

\begin{proof}[Proof of Lemma~\ref{lem:rl-target-entropy}]
Since $p^+$ and $p^-$ have disjoint supports and
$p=\rho p^++(1-\rho)p^-$, using $0\log0=0$ gives
\begin{align*}
 H(p)
 &=-\rho\sum_y p^+(y)\log\bigl(\rho p^+(y)\bigr)
   -(1-\rho)\sum_y p^-(y)\log\bigl((1-\rho)p^-(y)\bigr)\\
 &=h_{\mathrm b}(\rho)+\rho H(p^+)+(1-\rho)H(p^-),
\end{align*}
which proves Eq.~\eqref{eq:rl-entropy-decomposition}.
For fixed $p^+$ and $p^-$ with $H(p^+)<H(p^-)$,
\begin{equation*}
 \frac{\mathrm d}{\mathrm d\rho}
 \bigl[\rho H(p^+)+(1-\rho)H(p^-)\bigr]
 =H(p^+)-H(p^-)<0.
\end{equation*}
\end{proof}

\section{Proof of the Conditional KL Estimator}
\label{app:ckl-estimator}

\begin{insight}
On successful responses, the sequence log-ratio is the sum of a
success-probability log-ratio and a conditional-distribution log-ratio.
Centering within successes removes the former exactly. Multiplying the
remaining signal by the policy score gives the conditional-KL gradient.
Using the same samples to estimate the mean introduces a finite-group
scaling factor, with zero signal for groups containing fewer than two successes.
\end{insight}

Fix a prompt $x$ and a finite response space with a deterministic binary
reward and a nonempty success set $\mathcal S^+$.
Write $p_\theta=\pi_\theta(\cdot\mid x)$, $p=p_{\theta_0}$,
$q=\pi_{\mathrm{ref}}(\cdot\mid x)$, $\rho_\theta=p_\theta(\mathcal S^+)$,
$\rho=\rho_{\theta_0}$, and $\rho_q=q(\mathcal S^+)$.
Assume $p_\theta$ is differentiable and both $p_\theta$ and $q$ are positive.
Superscript $+$ denotes conditioning on $\mathcal S^+$;
all gradients below are evaluated at $\theta_0$.

\begin{proof}[Proof of Theorem~\ref{thm:ckl-estimator}]
Let $d(y)=\log(p(y)/q(y))$, $\mu=\mathbb E_{p^+}[d(Y)]$, and
$K(\theta)=D_{\mathrm{KL}}(p_\theta^+\Vert q^+)$.
For $y\in\mathcal S^+$,
\begin{equation}
 d(y)=\log\frac{\rho}{\rho_q}+\log\frac{p^+(y)}{q^+(y)},
 \qquad
 d(y)-\mu=\log\frac{p^+(y)}{q^+(y)}-K(\theta_0).
 \label{eq:ckl-log-ratio-decomposition}
\end{equation}
The scores $s=\nabla\log p_\theta$ and $s^+=\nabla\log p_\theta^+$ satisfy
$s=s^++\nabla\log\rho_\theta$ on $\mathcal S^+$ and
$\mathbb E_{p^+}[s^+]=0$. Hence
\begin{align}
 \nabla K
 &=\mathbb E_{p^+}\!\left[
 s^+(Y)\left(\log\frac{p^+(Y)}{q^+(Y)}+1\right)\right]\notag\\
 &=\mathbb E_{p^+}[s^+(Y)(d(Y)-\mu)]
  =\mathbb E_{p^+}[s(Y)(d(Y)-\mu)].
 \label{eq:ckl-conditional-gradient}
\end{align}

Conditional on $N=n\ge2$, relabel the successful responses as
$Y_1,\ldots,Y_n\overset{\mathrm{iid}}{\sim}p^+$.
Set $\bar d=n^{-1}\sum_{j=1}^n d(Y_j)$,
$a=\mathbb E_{p^+}[d(Y)s(Y)]$, and $b=\mathbb E_{p^+}[s(Y)]$.
Independence gives
\begin{align}
 \mathbb E[\bar d\,s(Y_i)\mid N=n]
 &=\frac{a}{n}+\frac{n-1}{n}\mu b,\notag\\
 \mathbb E[(d(Y_i)-\bar d)s(Y_i)\mid N=n]
 &=\frac{n-1}{n}(a-\mu b)
  =\frac{n-1}{n}\nabla K.
 \label{eq:ckl-sample-centering-factor}
\end{align}
Summing over the $n$ successes and dividing by $G$ proves
Eq.~\eqref{eq:ckl-finite-gradient}; for $n\le1$, both sides are zero.
Multiplication by $-\beta$ gives the stated expected gradient contribution.
Finally, $(d_i+a_0)-(\bar d+a_0)=d_i-\bar d$ proves invariance
under a common shift of the successful log-ratios.
\end{proof}

\section{Supplementary Experiments on Proximal and Distal Distillation}
\label{app:distillation-experiments}

We report supplementary results for distillation from identical, nearby,
and more distant teachers.

\subsection{Experimental settings}

The proximal-limit experiment initializes both student and frozen teacher
from Qwen3-1.7B-Base~\citep{yang_qwen3_2025}. Nearby teachers are obtained by training the same
model with GRPO for 50, 100, or 200 steps. The distal experiments use a
JustRL teacher with R1-Distill-1.5B~\citep{deepseek-ai_deepseek-r1_2025} and Qwen2.5-Math-1.5B~\citep{qwen_qwen25_2025} students;
both use the same cached teacher responses for SFT. Training uses slime
with Megatron and SGLang at revision \texttt{a067ce6f}, with:
\begin{itemize}
  \item seed: 1;
  \item Qwen3 experiments: DAPO17k~\citep{yu_dapo_2025}, maximum response length 2,048;
  \item JustRL experiments: JustRL training data~\citep{he_justrl_2025}, maximum response length 8,192;
  \item learning rate: $10^{-5}$.
\end{itemize}

\paragraph{Distillation from nearby teachers.} With each of the three GRPO-trained teachers, OPD achieves a higher mean
evaluation score than SFT (Table~\ref{tab:distillation-performance}),
consistent with its proximal advantage in these single-seed experiments.
%

\paragraph{Distillation with a larger initial performance gap.} With the JustRL teacher, both methods bring R1-Distill-1.5B close to
teacher performance, with SFT slightly ahead. For Qwen2.5-Math-1.5B,
OPD rollout correctness declines and remains low
(Figure~\ref{fig:distillation-distal}), while SFT improves the mean
evaluation score from 15.27 to 47.46 (Table~\ref{tab:distillation-performance}).

\begin{figure}[htbp]
\centering
\includegraphics[width=\linewidth]{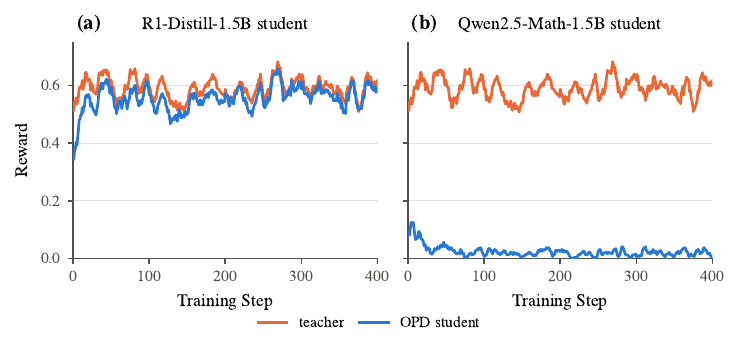}
\caption{Reward curve of distal OPD with the JustRL teacher and two initial students.
OPD rollout correctness increases for R1-Distill-1.5B but decreases
and remains low for Qwen2.5-Math-1.5B. Note that since we do not have student's on policy samples in SFT, we can not draw a similar reward curve for SFT. Their
evaluation results are reported in
Table~\ref{tab:distillation-performance}.}
\label{fig:distillation-distal}
\end{figure}

\clearpage
\subsection{Evaluation results}

\begin{table}[htbp]
\centering
\caption{Evaluation results of OPD and SFT with 16 sampled responses per problem. We can see that OPD typically outperforms SFT in proximal distillation, while it may experience severe instability issue and performance degradation in distal distillation. Proximal teachers are obtained by training the base model with GRPO for 50, 100, or 200 steps.}
\label{tab:distillation-performance}
\vspace{5pt}
\begingroup
\small
\setlength{\tabcolsep}{5pt}
\renewcommand{\arraystretch}{1.45}
\newcommand{\distillationgroup}[1]{%
  \specialrule{\lightrulewidth}{2pt}{0pt}%
  \multicolumn{7}{l}{\textcolor{frameworkink}{\textbf{#1}}}\\%
  \specialrule{\lightrulewidth}{0pt}{2pt}}
\newcommand{\distillationcell}[2]{%
  \multirow[c]{4}{#1}[-\dimexpr(\aboverulesep+\belowrulesep+\cmidrulewidth)/2\relax]{#2}}
  \begin{tabularx}{\linewidth}{>{\raggedright\arraybackslash}X l<{\hspace{14pt}} l r r r r}
\specialrule{\heavyrulewidth}{0pt}{0pt}
\rowcolor{frameworktint}
\textcolor{frameworkink}{\textbf{Student}} &
\textcolor{frameworkink}{\textbf{Teacher}} &
\textcolor{frameworkink}{\textbf{Model}} &
\textcolor{frameworkink}{\textbf{AIME-24}} &
\textcolor{frameworkink}{\textbf{AMC23}} &
\textcolor{frameworkink}{\textbf{MATH500}} &
\textcolor{frameworkink}{\textbf{Mean}} \\
\distillationgroup{Proximal Distillation}
\distillationcell{=}{Qwen3-1.7B-Base} & \distillationcell{*}{GRPO-50} & Student & 0.62 & 7.50 & 16.19 & 8.10 \\
 & & Teacher & 1.04 & 12.19 & 25.62 & 12.95 \\
\cmidrule(l){3-7}
 & & SFT & \textbf{1.88} & \underline{15.78} & \underline{32.69} & \underline{16.78} \\
 & & OPD & \underline{1.25} & \textbf{16.25} & \textbf{35.80} & \textbf{17.77} \\
\midrule
\distillationcell{=}{Qwen3-1.7B-Base} & \distillationcell{*}{GRPO-100} & Student & 0.62 & 7.50 & 16.19 & 8.10 \\
 & & Teacher & \textbf{1.67} & \underline{18.28} & 33.50 & 17.82 \\
\cmidrule(l){3-7}
 & & SFT & \underline{1.04} & \textbf{19.06} & \underline{39.32} & \underline{19.81} \\
 & & OPD & 0.62 & \textbf{19.06} & \textbf{42.58} & \textbf{20.75} \\
\midrule
\distillationcell{=}{Qwen3-1.7B-Base} & \distillationcell{*}{GRPO-200} & Student & 0.62 & 7.50 & 16.19 & 8.10 \\
 & & Teacher & 2.50 & 22.50 & 46.27 & 23.76 \\
\cmidrule(l){3-7}
 & & SFT & \textbf{3.33} & \underline{25.16} & \underline{49.73} & \underline{26.07} \\
 & & OPD & \underline{2.92} & \textbf{25.62} & \textbf{51.36} & \textbf{26.63} \\
\distillationgroup{Distal Distillation}
\distillationcell{=}{R1-Distilled-Qwen-1.5B} & \distillationcell{*}{JustRL-DS} & Student & 21.04 & 63.75 & 75.90 & 53.56 \\
 & & Teacher & \textbf{49.38} & 83.28 & 84.16 & \underline{72.27} \\
\cmidrule(l){3-7}
 & & SFT & \underline{47.92} & \textbf{85.16} & \textbf{85.74} & \textbf{72.94} \\
 & & OPD & 47.29 & \underline{83.44} & \underline{85.39} & 72.04 \\
\midrule
\distillationcell{=}{Qwen2.5-Math-1.5B} & \distillationcell{*}{JustRL-DS} & Student & 4.38 & 18.44 & 23.00 & 15.27 \\
 & & Teacher & \textbf{49.38} & \textbf{83.28} & \textbf{84.16} & \textbf{72.27} \\
\cmidrule(l){3-7}
 & & SFT & \underline{11.88} & \underline{56.09} & \underline{74.41} & \underline{47.46} \\
 & & OPD & 0.62 & 3.75 & 5.17 & 3.18 \\
\bottomrule
\end{tabularx}
\endgroup
\end{table}

%
%
\subsection{Supplimentary Experiments on Standard and Reversed Iterative Self-Distillation}\label{app:supp-exp-isd}

Here are some supplimentary experiments on two iterative self-distillation methods mentioned in the main body, See Figure~\ref{fig:supp-exp-isd}.

\begin{figure}[htb]
\centering
\includegraphics[width=\linewidth]{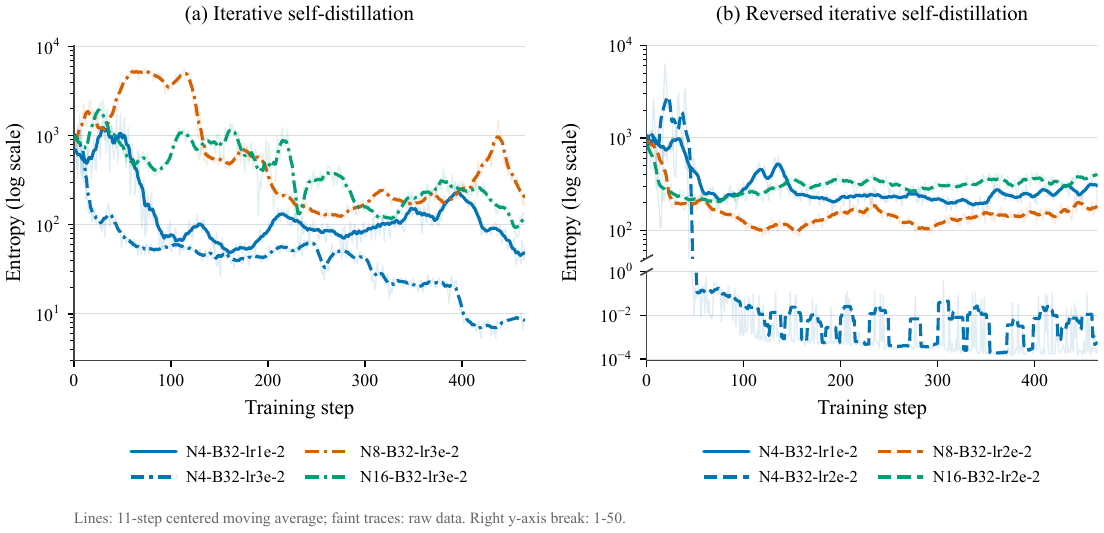}
\caption{Additional experiments on entropy dynamics of iterative self-distillation and reversed iterative self-distillation. In the legend ``N'' denotes the number of samples per request, ``B'' denotes batchsize and ``lr'' denotes the learning rate of SGD optimizer.}
\label{fig:supp-exp-isd}
\end{figure}

In a single training step, model entropy may randomly increase or decrease, but over the long run it exhibits a deterministic downward trend in expectation. These entropy curves satisfies the derivation in the main body very well. In fact the decreasing expected entropy does not guarantee the probability of decrement and increment. For example in reversed iterative self-distillation, since the punishment on samples with a high probability will raise the entropy, the entropy policy model has a high probability to be increased. However, the negative expectation indicates that once the decrement happens, the magnitude of the decrease will be substantial. This was also observed in our experiments.

These experiments share these experimental settings:

\begin{itemize}
\item seed: 1
\item total\_train\_epochs: 2
\item actor.path: Qwen/Qwen3-1.7B-Base
\item dataset: openai/gsm8k
\item actor.dtype: bfloat16
\item actor.kl\_ctl: 0
\item actor.optimizer.type: sgd
\item actor.optimizer.momentum: 0
\end{itemize}

Other hyperparameters are kept the same as the default value of AReaL, which can be found in our codebase.

\section{Supplimentary Experiments on Reinforcement Learning}\label{app:rl-exps}

\subsection{Additional Ablation Study on Advantage Estimation}

Here are some other supplimentary experiments on RL absent in the main body. They provide more evidence on the theoretical results obtained by our analysis and beyond.

\begin{figure}[htb]
  \centering
  \includegraphics[width=\linewidth]{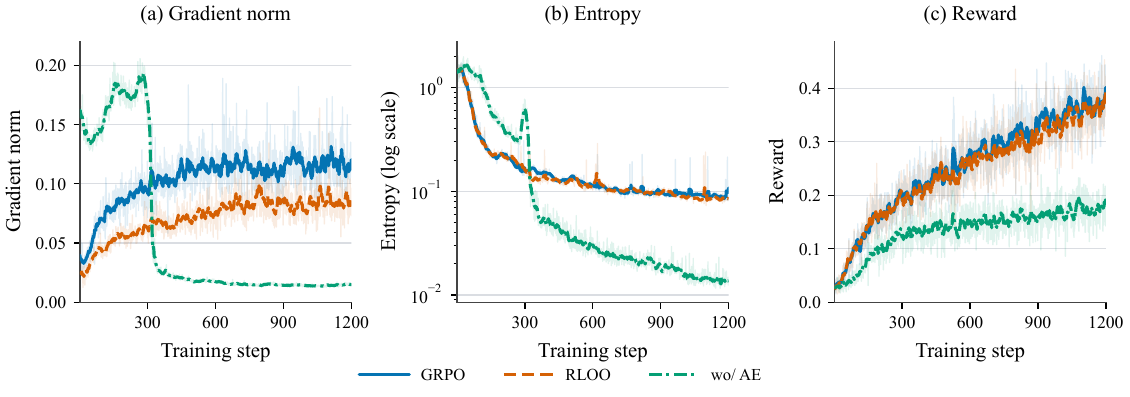}
  \caption{The ablation study on the advantage estimation on SmolLM2-360M-Instruct~\citep{allal_smollm2_2025} and GSM8K~\citep{cobbe_training_2021} dataset. with batchsize 128, n\-samples 4 and AdamW optimizer with learning rate 1e-5.}
  \label{fig:supp-exp-rl-smollm}
\end{figure}

In Figure~\ref{fig:supp-exp-rl-smollm} we also examined the variance effect on smaller models like SmolLM2-360M-Instruct. We used similar training settings like that in \citet{tajwar_maximum_2026}. We can see that advantage estimation methods effectively prevents entropy collapse and stabilizes the training process. These experiments share the following hyperparameter settings in AReaL~\citep{fu_areal_2025}, unmentioned hyperparameters are kept the same as the default setting:

\begin{itemize}
  \item seed: 1
  \item batchsize: 128
  \item n\_samples: 4
  \item total\_train\_steps: 1200
  \item actor.optimizer.lr: 1e-5
  \item actor.reward\_scaling: 2.0
  \item actor.reward\_bias: -0.5
\end{itemize}

\subsection{Additional Results on Conditional KL Restriction}

\begin{figure}[htb]
  \centering
  \includegraphics[width=\linewidth]{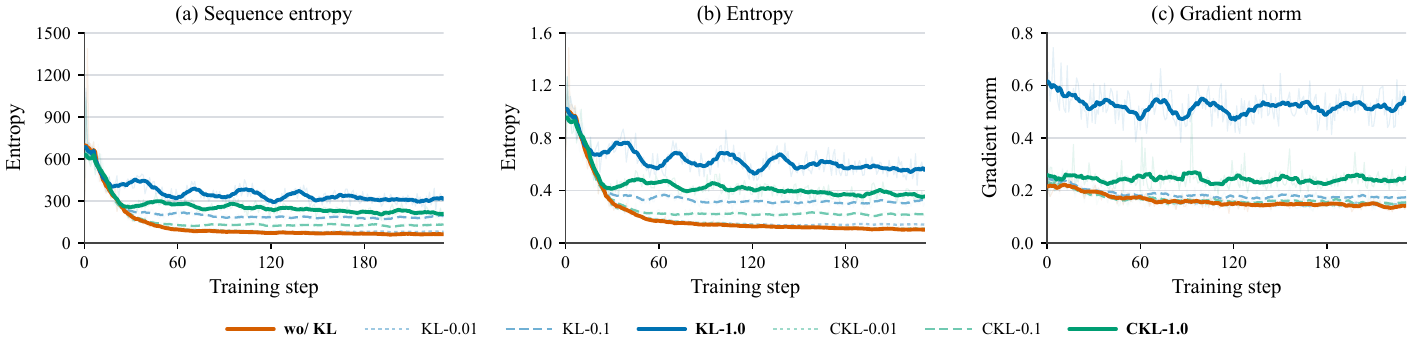}
  \caption{Additional metrics on the comparison of KL and CKL restrictions.}
  \label{fig:addi-kl-ckl}
\end{figure}

The results in Figure~\ref{fig:addi-kl-ckl} come from the same set of experiments as that in Figure~\ref{fig:kl-ckl-comparison}. We can see that the sequence-level entropy showed the same decreasing trend as that of token-level entropy. Recall the entropy decomposition in Lemma~\ref{lem:rl-target-entropy}, the entropy of a model in RL is only comparable if they have similar accuracy. Otherwise the entropy value may not directly reflect the collapse in diversity. The gradient norm behavior of these experiments also provides some explanation on the advantage of CKL restriction. With high restriction strength, the gradient of CKL is much more modest while it still preserves the diversity very well.

These experiments adopts these hyperparameter settings, while unmentioned ones use the default value in AReaL~\citep{fu_areal_2025}:

\begin{itemize}
  \item seed: 1
  \item batchsize: 128
  \item n\_samples: 4
  \item total\_train\_epochs: 4
  \item actor.optimizer.lr: 1e-6
  \item actor.reward\_scaling: 2.0
  \item actor.reward\_bias: -0.5
  \item actor.kl\_estimator: ``k1''
\end{itemize}

\begin{figure}[htb]
  \centering
  \includegraphics[width=\linewidth]{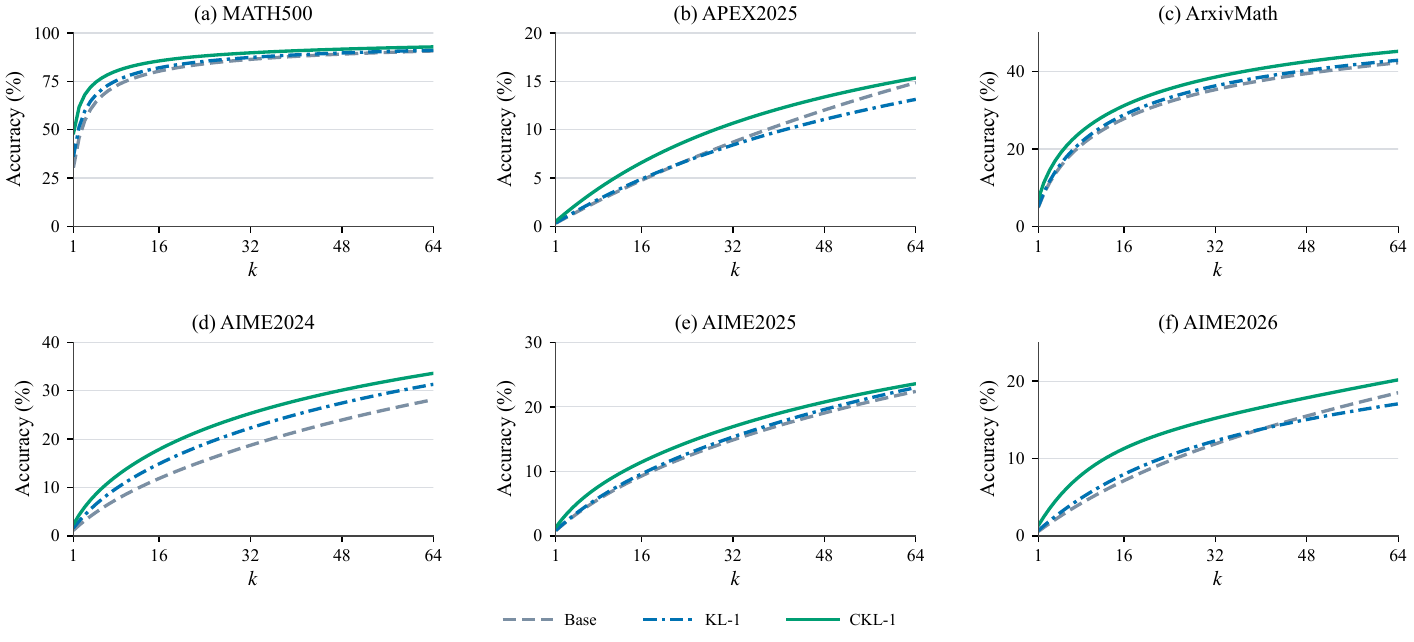}
  \caption{Pass@k performance of models trained for 4 epochs with KL or CKL regularization on MATH500~\citep{lightman_lets_2023}, Apex2025~\citep{balunovic_matharena_2025}, arxivmath~\citep{balunovic_matharena_2025}, and AIME~\citep{balunovic_matharena_2025} datasets. CKL regularization also preserves more learning potential in the sense of pass@k performance compared with KL regularization with the same strength. The evaluation settings is the same as that used in AReaL training pipeline.}
  \label{fig:passk-eval-full}
\end{figure}

\section{Further Discussions}
\label{app:further-discussions}

Our analysis of variance accumulation concerns entropy evolution during
prolonged language-model training. More broadly, it raises the possibility
that stochastic fluctuations can induce a deterministic drift toward lower
diversity in other iterative systems, including biological evolution and
technological development. For example, even when two technological
solutions are equally effective, random fluctuations in their early
exploration and adoption could become self-reinforcing, leaving one
insufficiently explored and producing a less diverse outcome. Such a
mechanism could contribute to path dependence without requiring an
intrinsic difference in the quality of the alternatives.

These connections remain speculative. Establishing whether variance
accumulation operates in such systems, and distinguishing its effects
from other mechanisms, would require domain-specific model and empirical
evidence that are difficult to obtain and beyond the scope of this work.
We mention these possibilities only to suggest broader directions for
investigating the phenomenon.

\end{document}